\documentclass[11pt]{article}
\usepackage[margin=1in]{geometry}

\usepackage{tikz}

\usepackage[round]{natbib}
\usepackage{amsfonts}       
\usepackage{nicefrac}       
\usepackage{amsmath,amssymb,amsthm}

\usepackage[colorlinks,allcolors=blue]{hyperref}

\usepackage{algorithmic,algorithm}
\usepackage{float}
\usepackage{mathtools,thmtools}
\usepackage[title]{appendix}
\usepackage{ dsfont }
\usepackage[capitalize]{cleveref}
\usepackage{ifthen}
\usepackage{xcolor}

\theoremstyle{plain}
\newtheorem{theorem}{Theorem}
\newtheorem{lemma}{Lemma}

\newtheorem{corollary}{Corollary}

\newtheorem{remark}[theorem]{Remark}

\declaretheoremstyle[ 
        spaceabove=\topsep, 
        spacebelow=\topsep, 
        headfont=\normalfont\itshape,
        bodyfont=\normalfont,
        notefont=\normalfont\itshape,
        notebraces={}{},
        postheadspace=0.33em, 
        qed=$\square$, 
        headpunct={.},
    ]{proofstyle}
\declaretheorem[style=proofstyle,numbered=no,name=Proof]{proof}

\usepackage[textsize=tiny]{todonotes}

\usepackage{makecell}
\usepackage{enumitem}

\newcommand{\logterm}{\iota}
\newcommand{\logtermval}{\log \frac{HSAK}{\delta}}

\newcommand{\calS}{\mathcal{S}}
\newcommand{\calA}{\mathcal{A}}
\newcommand{\calM}{\mathcal{M}}
\newcommand{\calP}{\mathcal{P}}

\newcommand{\calF}{\mathcal{F}}

\newcommand{\bbI}{\mathbb{I}}
\newcommand{\bbR}{\mathbb{R}}
\newcommand{\bbE}{\mathbb{E}}

\newcommand{\sinit}{s_{init}}
\newcommand{\ol}[1]{\overline{#1}}
\newcommand{\ul}[1]{\underline{#1}}

\newcommand{\occ}{\mu}

\newcommand{\aggQ}{U}
\newcommand{\vp}{W}

\renewcommand{\l}{\left}
\renewcommand{\r}{\right}

\DeclarePairedDelimiterX{\infdiv}[2]{(}{)}{%
  #1\;\delimsize\|\;#2%
}

\DeclareMathOperator*{\argmin}{arg\,min}

\newcommand{\brackets}[4][a]{%
  \ifx a#1\l#2 #3 \r#4\else
  \ifx s#1#2 #3 #4\else
  \ifx b#1\big#2 #3 \big#4\else
  \ifthenelse{\equal{#1}{B}} {\Big#2 #3 \Big#4} {}
  \ifthenelse{\equal{#1}{Bg}} {\Bigg#2 #3 \Bigg#4} {}
  \fi\fi\fi
}

\title{Near-optimal Regret Using Policy Optimization in Online MDPs with Aggregate Bandit Feedback}

\author{
Tal Lancewicki\footnote{Blavatnik School of Computer Science, Tel Aviv University; \texttt{lancewicki@mail.tau.ac.il}.}
\and
Yishay Mansour\footnote{Blavatnik School of Computer Science, Tel Aviv University and Google Research; \texttt{mansour.yishay@gmail.com}.}
}

\begin{document}
\maketitle



\begin{abstract}
    We study online finite-horizon Markov Decision Processes with adversarially changing loss and aggregate bandit feedback (a.k.a full-bandit). Under this type of feedback,
    the agent observes only the total loss incurred over the entire trajectory, rather than the individual losses at each intermediate step within the trajectory. We introduce the first Policy Optimization algorithms for this setting.
    In the known-dynamics case, we achieve the first \textit{optimal} regret bound of $\tilde \Theta(H^2\sqrt{SAK})$, where $K$ is the number of episodes, $H$ is the episode horizon, $S$ is the number of states, and $A$ is the number of actions. 
    In the unknown dynamics case we establish regret bound of $\tilde O(H^3 S \sqrt{AK})$, significantly improving the best known result by a factor of $H^2 S^5 A^2$.
\end{abstract}

\section{Introduction}

The standard model of reinforcement learning (RL) assumes a
rich feedback loop, where for each step within the episode the agent observes the loss in that state as feedback. While ideal, this is often  not the case in real-world applications. For example, in multi-turn dialogues with an LLM, feedback is typically available only at the end of the entire dialogue, not for each intermediate response. Similarly, in robotic manipulation, feedback is often only available for the entire trajectory, indicating whether the robot successfully completed its task, rather than providing feedback at every step of the robot's movement.

To address this challenge, \citet{efroni2021reinforcement} have initiated the study of aggregate bandit feedback in a stochastic setting where losses are generated in an i.i.d. manner. More recently, \citet{cassel2024near} have extended this to linear MDPs. Most related to our work is that of \citet{cohen2021online}, who considered the setting we study here, where the losses are non-stochastic and may be chosen by an adversary. 
They introduce a variant of bandit linear optimization which they call Distorted Linear Bandits and provide a solution for it. Through the framework of occupancy measures, they reduce the problem of adversarial MDPs with aggregate bandit feedback to Distorted Linear Bandits. This approach allows them to obtain a regret of $\tilde O(H^5 S^6 A^{5/2} \sqrt{K})$, where $K$ is the number of episodes, $H$ is the horizon, $S$ is the number of states and $A$ is the number of actions. While the dependency on $K$ is optimal, the dependency on other parameters is far from optimal.

In this paper, we revise the setting considered in \citet{cohen2021online}, both in the known and unknown dynamics cases. We present algorithms based on the Policy Optimization framework \citep{cai2020provably,shani2020optimistic,luo2021policy,chen2022policy}, which has strong connections to many practical algorithms such as NGP \citep{kakade2001natural}, TRPO \citep{schulman2015trust}, and PPO  \citep{schulman2017proximal}. Our algorithms have a closed-form update and are more efficient than that of \citet{cohen2021online}, which requires solving a convex optimization problem in each iteration. We obtain the first optimal bound under known dynamics and significantly improve the regret bound of \citet{cohen2021online} in the unknown dynamics case.


\paragraph{Summary of Contributions. \hspace{0em}}
The main contributions of the paper are as follows:
\begin{itemize}
    \item We present the first Policy Optimization algorithms for Online MDPs with aggregate bandit feedback.
    \item Under known dynamics we are the first to establish the near-optimal regret bound of $\tilde \Theta(H^2 \sqrt{S A K})$.
    \item In the unknown dynamics case, we achieve a regret of $\tilde{O}(H^3 S \sqrt{AK})$. Surprisingly, this regret bound matches the best known regret for Policy Optimization with semi-bandit feedback \citep{luo2021policy}.\footnote{
        \citet{luo2021policy} presents a slightly different dependence on the horizon $H$. This is due to their assumption of a loop-free MDP, which effectively enlarges their state space by a factor of $H$ compared to our model — see \Cref{remark:loop-free}.
    } 
    \item We establish a new lower bound $\Omega(H^2 \sqrt{S A K})$ for online MDPs with aggregate bandit feedback. To the best of our knowledge, this is the first lower bound for this setting that is not directly implied from the semi-bandit case. 
\end{itemize}
A comparison of our results to previous works is summarized in \Cref{table: comparison}.

\paragraph{Overview of techniques.}
Much of our algorithms' design and analysis follows the seminal work of \citet{luo2021policy}. As most of the regret minimization literature, they built upon the fundamental value difference lemma \citep{even2009online} that breaks the total regret as a weighted sum of local regrets in each state separately with respect to the $Q$-function. Our central observation is that the regret can be decomposed in a similar manner but with respect to different quantities, which we call the $U$-values. These quantities are particularly natural in the context of aggregate bandit feedback. The $U$-function at a given state and action is the expected cost \textbf{on the entire trajectory} given that we visit this state-action pair. 
Our decomposition with respect to the $U$-function is especially useful in the setting of aggregate bandit feedback since the $U$-function can be easily estimated using only the accumulated trajectory loss. This is in contrast to $Q$-function estimation that typically uses individual losses or the cost-to-go from a state-action pair.


\subsection{Related work}
\paragraph{\bf Aggregate bandit feedback with stochastic i.i.d losses\hspace{-0.5em}} was first studied by \citet{efroni2021reinforcement} who obtained regret of $\tilde O(H^{3/2}S^2 A^{3/2} \sqrt{K})$ with an efficient algorithm. Their transition and loss functions are not horizon-dependent. Adapting their bound for horizon-dependent losses and transitions as we consider here would effectively inflate the number of states by a factor of $H$, resulting in a regret bound of $\tilde O(H^{7/2} S^2 A^{3/2} \sqrt{K})$. \citet{cassel2024near} introduced the first algorithm for Linear MDPs with stochastic losses and aggregate bandit feedback. Their algorithm attains regret of $\tilde O(\sqrt{d^5 H^7 K})$ where $d$ is the dimension of the feature map. In the special case of tabular MDPs, they show a regret of $\tilde O (H^{7/2} S^2 A^{3/2} \sqrt{K})$.

\paragraph{\bf Adversarial Linear bandits\hspace{-0.5em}} (see for example \citet{lattimore2020bandit}) is a variant of the classical Multi-armed Bandit problem where each action is associated with a vector in $\bbR^d$. The loss in each round is the inner product of the action with an unknown parameter vector chosen by an adversary. Through the concept of \textit{occupancy measures}, online MDPs with aggregate bandit feedback and known dynamics can be seen as a special case of Adversarial Linear bandits. In terms of regret bounds,
EXP2 \citep{dani2007price,cesa2012combinatorial} with a specific exploration distribution achieves the optimal bound of $\Theta(B \sqrt {dK \log N})$ for any finite set of $N$ actions in $\bbR^d$, where $B$ is a bound on the losses \citep{bubeck2012towards}. 
Using a discretization argument, this bound can be extended to $\tilde O(B d \sqrt{K})$ for any compact convex set. However, the EXP2 algorithm is not efficient in general.
The latter bound for general convex set is attainable with efficient algorithms (polynomial in $d$) under mild assumptions \citep{hazan2016volumetric}. When using occupancy measures to reduce online MDPs with aggregate bandit feedback to Linear bandits, the decision set is of dimension $d = HSA$, the bound of the loss in each round is $B=H$ and the number of deterministic policies is $N = A^{SH}$. This results in regret of $\tilde O(H^2 S \sqrt{A K})$ and $\tilde O(H^2 S A \sqrt{K})$ with inefficient and efficient algorithms, respectively. For the known dynamics case, we improve these bounds to optimal $\tilde \Theta(H^2  \sqrt{S A K})$ regret with an efficient and more natural algorithm.

As mentioned before, \citet{cohen2021online} have extended the linear bandit model to \textit{Distorted Bandit Online Linear Optimization} (DBOLO). They show that online MDPs with aggregate bandit feedback and \textit{unknown} dynamics can be reduced to DBOLO efficiently. Their algorithm is built upon the SCRIBLE algorithm \citep{abernethy2008competing} and guarantees regret of $\tilde O(H^5 S^6 A^{5/2} \sqrt{K})$. On the same setting, we improve their regret bound to $\tilde O(H^3 S \sqrt{A K})$.
\paragraph{\bf Regret minimization in MDPs with semi-bandit feedback\hspace{-0.5em}} is extensively studied in the literature, initiated with the seminal UCRL algorithm \citep{jaksch2010near} for stochastic losses. Their model was later extended to the more general Online (adversarial) MDPs where the loss functions are arbitrarily chosen by an adversary. Most algorithms for this model are based either on the framework of occupancy measures  \citep{zimin2013online,rosenberg2019bandit,rosenberg2019online,jin2019learning} or the Policy Optimization framework \citep{even2009online,shani2020optimistic,luo2021policy}. In the adversarial model with semi-bandit feedback and known dynamics, the optimal regret bound is $\tilde \Theta (H\sqrt{S A K})$ case and is attained by an occupancy-measure based algorithm \citep{zimin2013online}. With PO, the state-of-the-art regret under known dynamics is $\tilde O (H^{2}\sqrt{S A K})$. We achieve the same bound with aggregate bandit feedback which in this case is optimal. 
Under unknown dynamics, the best known bound is $\tilde O (H^{2} S\sqrt{A K})$ and is also attained by an occupancy-measure based algorithm \citep{jin2019learning}, while the best known lower bound is $\Omega(H^{3/2} \sqrt{S A K})$ \citep{jin2018q}. With policy optimization algorithm, the best known bound is $\tilde O (H^{3} S\sqrt{A K})$ \citep{luo2021policy}. Although we are in a setting with less informative feedback, we match the latter bound.

\begin{remark}
    \label{remark:loop-free}
    We note that some of the literature on semi-bandit feedback, such as \citet{jin2019learning, luo2021policy}, assumes \textit{loop-free} MDPs. Under this assumption the state space consists of $H$ disjoint sets $\calS = \calS_1 \cup \calS_2 \cup \dots \cup \calS_H$ such that in step $h$ the agent can only be found in states from the set $\calS_h$.
    Effectively, this means that their state space is larger than ours by a factor of $H$.
    So for example the regret bound $\tilde O(H^2 S \sqrt{A K})$ in \citet{luo2021policy} implies a bound of $\tilde O(H^3 S \sqrt{A K})$ in the transition model presented in this paper.
    We emphasize that these differences are rather artificial and not due to an actual difference in the regret.
\end{remark}

\paragraph{\bf Stochastic binary trajectory feedback\hspace{-0.5em}} was studied in \citet{chatterji2021theory}. In their model, the  rewards are drawn
from a logistic model that depends on features of the trajectory.

\paragraph{\bf Preference-based RL (PbRL)\hspace{-0.5em}} is a model where the feedback is given in terms of preferences over a trajectory pair instead of rewards. This is partially related to our model in motivation. A partial list of works on PbRL includes \citep{saha2023dueling,chen2022human,wu2023making}. For additional related work on the topic, see the above references.

\begingroup  
\renewcommand{\thefootnote}{\fnsymbol{footnote}}  

\begin{table}
    \caption{Comparison of regret bounds for online MDPs with aggregate bandit feedback. 
    The regret bounds presented in this table ignore logarithmic and low-order terms.}
    \begin{center}
        \begin{tabular}[c]{|l|c|c|l|}
            \hline
            Algorithm & Dynamics & Loss & \makecell{Regret} 
            \\ 
            \hline \hline
            \makecell[l]{Reduction to Efficient\\ Linear Bandits algorithm} & known & adversarial & $ \sqrt{H^{4} S^2 A^2 K}$
            \\
            \hline
            \makecell[l]{\Cref{alg:tabular-known-p main} \textbf{(ours)}} & known & adversarial & $ \sqrt{H^{4} S A K}$
            \\
            \hline
            \makecell[l]{Lower bound} & known & adversarial & $ \sqrt{H^{4} S A K}$
            \\
            \hline \hline 
            \makecell[l]{UCBVI-TS \\ \citep{efroni2021reinforcement}} & unknown & stochastic & $ \sqrt{H^{7} S^{4} A^{3} K}$ \footnotemark[3]
            \\
            \hline 
            \makecell[l]{REPO for tabular MDPs \\ \citep{cassel2024near}} & unknown & stochastic & $ \sqrt{H^{7} S^{4} A^{3} K}$
            \\
            \hline \hline
            \makecell[l]{Reduction to DBOLO \\ \citep{cohen2021online}} & unknown & adversarial & $ \sqrt{H^{10} S^{12} A^{5} K}$
            \\
            \hline
            \makecell[l]{\Cref{alg:tabular-unknown-p main} \textbf{(ours)} } & unknown & adversarial & $ \sqrt{H^{6} S^{2} A K}$
            \\
            \hline
        \end{tabular}
        \label{table: comparison}
    \end{center}

\end{table}
\footnotetext[3]{This regret bound is adapted to horizon-dependent transition and losses - see the related work section for more details.}

\endgroup

\section{Preliminaries}

A finite-horizon episodic adversarial MDP is defined by a tuple $\calM = (\calS , \calA , H ,\sinit, p, \{ \ell^{k} \}_{k=1}^K)$, where
$\calS$ and $\calA$ are state and action spaces of sizes $|\calS| = S$ and $|\calA| = A$, respectively, $H$ is the horizon, $\sinit$ is the initial state and $K$ is the number of episodes. 
$p: \calS \times \calA \times [H] \to \Delta_{\calS}$ is the \textit{transition function} such that $p_h(s' | s,a)$ is the probability to move to $s'$ when taking action $a$ in state $s$ at time $h$. 
$\{ \ell^{k} : \calS \times \calA \times [H] \to [0,1] \}_{k=1}^K$ are \textit{cost functions} chosen by an \textit{oblivious adversary}, where $\ell_h^k(s,a)$ is the cost for taking action $a$ at
$(s,h)$ in episode $k$.

A \textit{policy} $\pi: \calS \times [H] \to \Delta_{\calA}$ is a function that gives the probability $\pi_h(a | s)$ to take action $a$ when visiting state $s$ at time $h$.
The value $V^{\pi}_h(s ; \ell)$ is the expected cost of $\pi$ with respect to cost function $\ell$ starting from $s$ in time $h$, i.e.,
$V_{h}^{\pi}(s ; \ell) = \bbE \bigl[ \sum_{h'=h}^{H} \ell_{h'}(s_{h'},a_{h'}) | \pi , s_{h}=s \bigr]$,
where the expectation is with respect to policy $\pi$ and transition function $p$,
that is, $a_{h'} \sim \pi_{h'}(\cdot | s_{h'})$ and $s_{h'+1} \sim p_{h'}(\cdot | s_{h'},a_{h'})$.
The \textit{Q-function} is defined by  $Q_{h}^{\pi}(s,a ; \ell) = \bbE \bigl[ \sum_{h'=h}^{H} \ell_{h'}(s_{h'},a_{h'}) | \pi , s_{h}=s, a_{h}=a \bigr]$.
The occupancy measure $\occ^\pi_h(s,a) = \Pr[s_h=s,a_h=a | \pi,s_1=\sinit]$ is the distribution that policy $\pi$ induces over state-action pairs in step $h$, and we denote $\occ^\pi_h(s) = \sum_{a \in \calA} \occ^\pi_h(s,a)$.

\paragraph{Learning protocol and Feedback.}
The learner interacts with the environment for $K$ episodes.
At the beginning of episode $k$, it picks a policy $\pi^k$, and starts in an initial state $s^k_1 = \sinit$.
In each time $h\in [H]$, it observes the current state $s^k_h$, draws an action from the policy $a^k_h \sim \pi^k_h(\cdot | s_h^k)$ and transitions to the next state $s^k_{h+1} \sim p_h(\cdot | s^k_h,a^k_h)$. There are three types of loss feedback that are common in the literature:

\begin{itemize}
    \item In \textit{full-information} feedback, at the end of episode $k$ the agent observes the full cost function $\ell^k \in [0,1]^{H S A}$.
    \item Under \textit{bandit} feedback (a.k.a semi-bandit), the agent observes the loss function over the agent's trajectory, $\{\ell_h^k(s_h^k,a_h^k)\}_{h=1}^H$.
    \item Under \textit{aggregate bandit} feedback (a.k.a full-bandit), the agent observes only the entire episode loss, $L^k_{1:H} = \sum_{h=1}^H\ell_h^k(s_h^k,a_h^k)$. 
\end{itemize}

In all three settings, the trajectory $\{s_h^k, a_h^k\}_{h=1}^H$ is assumed to be fully observed.
In this work, we assume aggregate bandit feedback, which is the least informative among the three above.

The goal of the learner is to minimize the \textit{regret}, defined as the difference between the learner's cumulative expected cost and the best fixed policy in hindsight:
$$
    R_K
    =
    \sum_{k=1}^K V^{\pi^k}_1(\sinit;\ell^k) -  \sum_{k=1}^K V^{\pi^\star}_1(\sinit;\ell^k),
$$
where $\pi^\star = \argmin_{\pi} \sum_{k=1}^K V^{\pi}_1(\sinit;\ell^k)$ is the best fixed policy in hindsight.

\paragraph{Value difference lemma.} 
The analysis of policy optimization algorithms is often built upon a fundamental regret decomposition that follows the following lemma:

\begin{lemma}[Value Difference Lemma \citep{even2009online}]
\label{lemma: value diff}
    For any loss function $\ell$ and any policies $\pi$ and $\pi'$,
    \begin{align}
    \nonumber
        & V^{\pi}_1(\sinit; \ell) - V^{\pi'}_1(\sinit; \ell) 
        \\
        & \quad = \sum_{h,s} \occ_h^{\pi'} (s)  \l\langle \pi_h(\cdot \mid s) - \pi'_h(\cdot \mid s), Q_h^\pi(s,\cdot;\ell) \r\rangle.
        \label{eq: value diff}
    \end{align}
    where $\langle \cdot , \cdot \rangle$ is the inner product.
\end{lemma}


As a direct consequence we can break the regret as,
    \begin{align*}
        R_K = \sum_{h,s} \occ_h^{\pi^\star} (s) \sum_{k=1}^K\langle\pi_h^k(\cdot \mid s) - \pi_h^\star(\cdot \mid s), Q_h^k(s,\cdot)\rangle,
    \end{align*}
    where, for each $h$ and $s$, the internal sum over $k$ can be seen as regret of a Multi-armed bandit problem with respect to the loss vectors $Q_h^k(s,\cdot)$.

\textbf{Additional notations.}
Episode indices appear as superscripts and in-episode steps as subscripts. 
The notations $\tilde O(\cdot)$ and $\lesssim$ hide poly-logarithmic factors including $\log (K/\delta)$ for confidence parameter $\delta$.
$[n] = \{1,2,\dots,n\}$. The indicator of event $E$ is $\bbI\{E\}$ and we denote $\bbI_h^k(s,a) = \bbI\{s_h^k =s, a_h^k =s\}$.
We use the notations $V^k_h(s),Q^k_h(s,a),\occ^k_h(s,a)$ when the policy and cost are $\pi^k$ and $\ell^k$, respectively. The expectation conditioned on the policy $\pi^k$ is denoted by $\bbE_k$. 

\section{The $U$-function}
Policy Optimization algorithms build upon the Value difference lemma, but since the $Q$-function is unknown (due to the bandit feedback), one would need to estimate it. The state-of-the-art PO algorithm for semi-bandit case \citep{luo2021policy} estimates the $Q$-function via importance sampling:

\begin{align*}
    \hat{Q}_h^\pi(s,a;\ell) = \frac{\bbI\{s_h = s,a_h = s\} L_{h:H}}{\occ_h^\pi(s,a)}
\end{align*}
where $L_{h:H} = \sum_{h'=h}^H \ell_{h'}(s_{h'},a_{h'})$ is the realized loss-to-go from time $h$. To be more precise, \citep{luo2021policy} add a small bias at the denominator to better control its variance and use an upper confidence bound of $\occ_h^\pi(s,a)$ whenever the dynamics is unknown. Indeed $\hat{Q}_h^\pi(s,a;\ell)$ is an unbiased estimate of $Q_h^\pi(s,a;\ell)$.
Note that with aggregate bandit feedback $L_{h:H}$ cannot be computed and it becomes  unclear how to directly estimate the $Q$-function. For this reason we introduce a new quantity which we call the $U$-function. While the $Q_h^\pi(s,a;\ell)$ is the expected cost-to-go from time $h$ given that we visit state $s$ perform action $a$ at that time; the $U$-function is the expected cost \textbf{on the entire trajectory} given that we visit $s$ and perform $a$ at time $h$. That is,

\begin{align*}
    U_h^\pi(s,a;\ell) \coloneqq \bbE\l[ \sum_{h'=1}^H \ell_{h'}(s_{h'},a_{h'}) \Big\vert \pi, s_h = s, a_h = a\r].
\end{align*}

The following lemma shows that the difference  $U_h^\pi(s,a) - Q_h^\pi(s,a)$ does not depend on $a$.

\begin{lemma}
    \label{lemma:U-Q}
    For any Markovian policy $\pi$, loss function $\ell$ and $(h,s,a) \in [H]\times \calA\times \calS$,
    \begin{align*}
        U_h^\pi(s,a;\ell) - Q_h^\pi(s,a;\ell) = \vp_h^\pi(s;\ell)
    \end{align*}
    where $\vp_h^\pi(s;\ell) \coloneqq \bbE[ \sum_{h'=1}^{h-1} \ell_{h'}(s_{h'},a_{h'}) \vert \pi, s_h = s]$.
\end{lemma}

\begin{proof}
Due to the Markov property and the fact that $\pi$ is a Markovian policy, the trajectory up to time $h-1$ does not depend on the action taken at time $h$. Thus,
\begin{align*}
    U_h^\pi(s,a;\ell) - Q_h^\pi(s,a;\ell) 
    & 
    = \bbE\l[ \sum_{h'=1}^{h-1} \ell_{h'}(s_{h'},a_{h'}) \Big\vert \pi, s_h = s, a_h = a\r]
    \\
    & = \bbE\l[ \sum_{h'=1}^{h-1} \ell_{h'}(s_{h'},a_{h'}) \Big\vert \pi, s_h = s\r] = \vp_h^\pi(s;\ell).
    \\ &\qedhere
\end{align*}
\end{proof}

As a corollary of \cref{lemma:U-Q} we obtain the Value Difference Lemma with respect to the $U$-function.

\begin{corollary}
\label{lemma: value diff U}
    For any loss function $\ell$ and any policies $\pi$ and $\pi'$,
    \begin{align*}
        V^{\pi}_1(\sinit; \ell) - V^{\pi'}_1(\sinit; \ell) 
        = \sum_{h,s} \occ_h^{\pi'} (s)  \l\langle \pi_h(\cdot \mid s) - \pi'_h(\cdot \mid s), U_h^\pi(s,\cdot;\ell) \r\rangle . 
    \end{align*}
\end{corollary}
\begin{proof}
    For each $h$ and $s$, $\sum_a\pi_h(a \mid s)= \sum_a\pi'_h(a \mid s) = 1$. 
    Thus,
    \begin{align*}
         \l\langle \pi_h(\cdot \mid s) - \pi'_h(\cdot \mid s), U_h^\pi(s,\cdot;\ell) - Q_h^\pi(s,\cdot;\ell) \r\rangle 
         = \sum_{a\in\calA} (\pi_h(a \mid s) - \pi'_h(a \mid s))\vp_h^\pi(s;\ell) = 0
    \end{align*}
    Adding the above for each $h$ and $s$ in the sum on the right-hand side of \cref{eq: value diff} completes the proof.
\end{proof}
As in the $Q$-function case, a direct consequence is that we can break the regret as,
    \begin{align}
        R_K = \sum_{h,s} \occ_h^{\pi^\star} (s) \sum_{k=1}^K\langle\pi_h^k(\cdot \mid s) - \pi_h^\star(\cdot \mid s), U_h^k(s,\cdot)\rangle
        \label{eq: regret U}
    \end{align}
where similarly to other notations, we slightly abuse notation and write $U_h^k(s,a) = U_h^{\pi^k}(s,a;\ell^k)$.

In the context of aggregate bandit feedback, the $U$-function is much more useful as it can be estimated using the loss of the entire trajectory. In particular, the following is an unbiased estimator of ${U}_h^\pi(s,a;\ell)$:
\begin{align*}
    \hat{\aggQ}_h^\pi(s,a;\ell) = \frac{\bbI\{s_h = s,a_h = s\} L_{1:H}}{\occ_h^\pi(s,a)}
\end{align*}
where $L_{1:H} = \sum_{h=1}^H \ell_{h}(s_{h},a_{h})$ is the aggregated feedback of the trajectory. Later in the algorithms, we will use an optimistic variant of this estimator.

\section{Known Dynamics}

Our algorithm is based on the regret decomposition in \Cref{eq: regret U}. Each internal sum can be seen as regret of a bandit problem with loss vectors $U_h^k(s,\cdot)$. Thus, we run a  Multiplicative Weight Update with respect to an estimate of the $U$-function that uses only the aggregated trajectory loss: $\pi^{k+1}_h(\cdot\mid s) \propto \exp(\eta(\hat U_h^k (s,\cdot) - B_h^k (s,\cdot)))$. Here $\eta$ is a learning rate, $\hat U^k$ is the estimate of the $U$-function and $B^k$ is some bonus function that we'll define later. The estimate of the $U$-function is defined by,

\begin{align}
    \hat\aggQ_{h}^{k}(s, a)  = \frac{\mathbb{I}_{h}^{k}(s, a)}{\occ_{h}^{k}(s, a) + \gamma} L^{k}_{1:H}.
    \label{eq:U estimate}
\end{align}
where $\gamma = \Theta (1 / \sqrt{K})$ used to reduce the variance of the estimator. We note that in the known dynamics case $\occ_{h}^{k}(s, a)$ can be easily computed using dynamic programming.

Note that the regret decomposition in \cref{eq: regret U} is averaged with respect to the state occupancy of $\pi^\star$, while the second moment of the estimator is roughly scaled inversely with the state occupancy of $\pi^k$. In order to control this distribution mismatch, we reduce from our  $U$-estimate a bonus term $B_h^k(s,a)$ similar to \citet{luo2021policy}.
$B_h^k(s,a)$ is essentially a $Q$-function with respect to a known loss function $b^k$. In the known dynamics case we set $b^k_h(s) = \sum_{a \in \calA} \frac{3 \gamma H \pi^k_h(a \mid s) }{\occ^k_h(s) \pi^k_h(a \mid s) + \gamma}$, where the reason for this particular choice will become clearer later in the analysis. $B_h^k(s,a)$ is computed using standard Bellman equations, which is simpler and more intuitive than the dilated version of \citet{luo2021policy}. 


\begin{algorithm}[t]
    \caption{Policy Optimization with Aggregated Bandit Feedback and Known Transition Function}  
    \label{alg:tabular-known-p main}
    \begin{algorithmic}
        \STATE \textbf{Input:} state space $\calS$, action space $\calA$, horizon $H$, learning rate $\eta > 0$, exploration parameter $\gamma > 0$, confidence parameter $\delta > 0$.
        
        \STATE \textbf{Initialization:} 
        Set $\pi_{h}^{1}(a \mid s) = {1}/{A}$ for every $(h,s,a) \in [H]\times \calS \times \calA$.
        
        \FOR{$k=1,2,\dots,K$}
            
            \STATE Play episode $k$ with policy $\pi^k$ and observe aggregated bandit feedback $L^k_{1:H} = \sum_{h=1}^H \ell_h^k(s^k_h,a^k_h) $.
            
            
            \STATE 
            $
                \hat\aggQ_{h}^{k}(s, a)  = \frac{\mathbb{I}_{h}^{k}(s, a)}{\occ_{h}^{k}(s, a) + \gamma} L^{k}_{1:H}
            $

            \STATE {\color{gray} \# Bonus Computation}
            \STATE Set $B^k_{H+1}(s,a) = 0$ for every $(s,a) \in \calS \times \calA$.
            
            \FOR{$h=H,H-1,\dots,1$}
                
                \FOR{$(s,a) \in \calS \times \calA$}
                
                    \STATE $b^k_h(s) = \sum_{a \in \calA} \frac{3 \gamma H \pi^k_h(a \mid s) }{\occ^k_h(s) \pi^k_h(a \mid s) + \gamma}$
                    \STATE $B_{h}^{k}(s,a) = b_{h}^{k}(s) +  {\sum_{s' ,a' } p_h(s'\mid s,a) \pi_{h+1}^{k}(a' \mid s') B^k_{h+1}(s',a')}$

                \ENDFOR
            \ENDFOR
            \STATE {\color{gray} \# Policy Improvement}
            \STATE For every $(s,a,h) \in \calS \times \calA \times [H]$:
            \[
                \pi^{k+1}_h(a \mid s) = \frac{\pi^k_h(a \mid s) e^{-\eta( \hat \aggQ_{h}^{k}(s,a) - B_h^k(s,a) )} } 
                {\sum_{a'} \pi^k_h(a' \mid s) e^{-\eta( \hat \aggQ_{h}^{k}(s,a') - B_h^k(s,a') ) }}.
            \]
        \ENDFOR
    \end{algorithmic}
\end{algorithm}

\begin{theorem}
    \label{theorem:known}
    Under known dynamics,
    running \Cref{alg:tabular-known-p main} with $\eta=({H \sqrt{S A K} + H^2\sqrt{K}})^{-1}$ and $\gamma = 2\eta H$, guarantees with probability $1-\delta$,
    \begin{align*}
        R_K \leq \tilde O (H^{2}\sqrt{SAK} 
        + H^{3}\sqrt{K}).
    \end{align*}
\end{theorem}

If $SA \geq H^2$ (which is typically the case since for any practical application $S \gg H$), the first term dominates and our bound is optimal up to logarithmic terms. A straightforward reduction of our problem to linear bandits that uses the efficient algorithm of \citet{hazan2016volumetric} guarantees expected regret of $\tilde O(H^2 S A \sqrt{K})$. Our algorithm improves that by a factor of $\sqrt{S A}$ and guarantees regret with high probability rather than only in expectation. In addition, our algorithm is more computationally efficient since it has a closed form update as opposed to the reduction that requires solving a convex optimization problem in each iteration.


Before we outline the proof of \Cref{theorem:known}, let us first show that $\hat \aggQ_h^k(s,a)$ is a nearly unbiased estimator of $U_h^k(s,a)$.

\begin{lemma}
    \label{lemma:unbiased}
    Under \Cref{alg:tabular-known-p main}, for any $h,s,a$ and $k$,
    \begin{align*}
        \mathbb{E}_k \l[ \hat\aggQ_{h}^{k}(s,a)\r]
        =
        \frac{\occ_{h}^{k}(s,a)}{\occ_{h}^{k}(s,a) + \gamma} \aggQ_{h}^{k}(s,a)
    \end{align*}
\end{lemma}

\begin{proof}
    By definition $\Pr(\bbI_h^k(s,a) = 1 \mid \pi^k) = \occ_h^k(s,a)$. Using the law of total expectation and the fact that $\hat\aggQ_{h}^{k}(s, a) = 0$ whenever $\bbI_h^k(s,a) = 0$ we get,
    \begin{align*}
    \mathbb{E}_k \l[\hat\aggQ_{h}^{k}(s, a) \r] 
     &= \mathbb{E}_k \l[\hat\aggQ_{h}^{k}(s, a) \mid \bbI_h^k(s,a) = 1 \r] 
        \cdot \occ_h^k(s,a)
          + \mathbb{E}_k \l[\hat\aggQ_{h}^{k}(s, a) \mid \bbI_h^k(s,a) = 0 \r] \cdot (1 -  \occ_h^k(s,a))
        \\
        & =
        \mathbb{E}_k \l[\frac{\sum_{h' = 1}^{H} \ell_{h'}^{k}(s_{h'}^{k}, a_{h'}^{k})}{\occ_{h}^{k}(s, a) + \gamma}  \mid \mathbb{I}_{h}^{k}(s, a) = 1 \r] \cdot \occ_h^k(s,a)
        \\
            &  = \frac{\occ_{h}^{k}(s, a)}{\occ_{h}^{k}(s, a) + \gamma} \mathbb{E}_k \l[\sum_{h' = h}^{H} \ell_{h'}^{k}(s_{h'}^{k}, a_{h'}^{k}) \,\Big\vert\, s_{h}^{k} = s, a_{h}^{k} = a \r]
            \\
                &  = \frac{\occ_{h}^{k}(s, a)}{\occ_{h}^{k}(s, a) + \gamma}U_{h}^{k}(s, a).
                \\
                & \qedhere
    \end{align*}
\end{proof}
Given our novel regret decomposition in \Cref{eq: regret U} and the lemma above, the rest of the analysis follows similar steps as those in \citep{luo2021policy}.

\begin{proof}[sketch of \Cref{theorem:known}]
    Using \Cref{eq: regret U}, we can break the regret of the algorithm as follows:
\begin{align}
    \nonumber
        & \underbrace{ \sum_{k,h, s}\occ_{h}^{\star}(s) \l\langle \pi_{h}^{k}(\cdot\mid s), \aggQ_{h}^{k}(s, \cdot) - \hat\aggQ_{h}^{k}(s, \cdot) \r\rangle }_{\textsc{Bias}_{1}}
        + \underbrace{\sum_{k,h, s}\occ_{h}^{\star}(s) \l\langle \pi_{h}^{\star}(\cdot\mid s), \hat\aggQ_{h}^{k}(s, \cdot) - \aggQ_{h}^{k}(s, \cdot) \r\rangle }_{\textsc{Bias}_{2}}
        \\
        &  + \underbrace{\sum_{k,h, s}\occ_{h}^{\star}(s) \l\langle \pi_{h}^{k}(\cdot\mid s) - \pi_{h}^{\star}(\cdot\mid s), \hat\aggQ_{h}^{k}(s, \cdot) - B_{h}^{k}(s, \cdot) \r\rangle }_{\textsc{Reg}}
        + 
        \underbrace{\sum_{k,h, s}\occ_{h}^{\star}(s) \l\langle \pi_{h}^{k}(\cdot\mid s) - \pi_{h}^{\star}(\cdot\mid s), B_{h}^{k}(s, \cdot) \r\rangle }_{\textsc{Bonus}},
        \label{eq:regret decomposition}
\end{align}
From \Cref{lemma:unbiased} we immediately get that $\bbE_k[\textsc{Bias}_2] \leq 0$. Since $\|\hat\aggQ_{h}^{k}\|_\infty \leq H/\gamma$ we can also bound with high probability $\textsc{Bias}_2 \leq \tilde{O}(H^2 / \gamma)$ using standard concentration bounds (see \Cref{lemma:good bias2}). Again, using \Cref{lemma:unbiased} the expectation of $\textsc{Bias}_1$ equals the following, 
\begin{align*}
          \sum_{k, h, s, a}\occ_{h}^{\star}(s)\pi_{h}^{k}(a\mid s) \l(\aggQ_{h}^{k}(s, a) -  \mathbb{E}_{k} \l[\hat\aggQ_{h}^{k}(s, a) \r] \r)
            &  = \sum_{k, h, s, a}\occ_{h}^{\star}(s)\pi_{h}^{k}(a\mid s)\aggQ_{h}^{k}(s, a) \l(1 - \frac{\occ_{h}^{k}(s, a)}{\occ_{h}^{k}(s, a) + \gamma} \r)
            \\
            & = \sum_{k = 1}^{K} \sum_{h, s, a}\occ_{h}^{\star}(s) 
            \aggQ_{h}^{k}(s, a) \frac{\gamma \pi_{h}^{k}(a\mid s)}{\occ_{h}^{k}(s, a) + \gamma}
\end{align*}
Bounding $U_h^k(s,a)$ by $H$ we get that $\bbE_k[\textsc{Bias}_1] \leq \frac{1}{3}\sum_{k = 1}^{K} \sum_{h, s}\occ_{h}^{\star}(s)b_{h}^{k}(s)$. Using a form of Freedman's inequality that takes into account the second moment of the estimator, we can also bound $\textsc{Bias}_1 - \bbE_k[\textsc{Bias}_1] \leq \frac{1}{3}\sum_{k = 1}^{K} \sum_{h, s}\occ_{h}^{\star}(s)b_{h}^{k}(s) + \tilde O(H^2 / \gamma)$. In total we get that,
\begin{align*}
    \textsc{Bias}_1 \leq \underbrace{\frac{2}{3}\sum_{k = 1}^{K} \sum_{h, s}\occ_{h}^{\star}(s)b_{h}^{k}(s)}_{(i)} + \tilde O(H^2 / \gamma)
\end{align*}
Term $(i)$ is challenging due to the distribution mismatch between $\occ_{h}^{\star}(s)$ and $\occ_{h}^{k}(s)$ in the denominator of $b_h^k(s)$. However, we will later see that the $\textsc{Bonus}$ term will cancel it out, essentially allowing us to replace $\occ_{h}^{\star}(s)$ with $\occ_{h}^{k}(s)$.

Before we turn to the $\textsc{Bonus}$ term, let's consider $\textsc{Reg}$. Using the standard entropy-regularized OMD guarantee (\cref{lemma:OMD}), $\textsc{Reg}$ can be bounded by,
    \begin{align*}
        &\frac{H\ln A}{\eta} + 2\eta\!\!\sum_{k,h, s, a}\occ_{h}^{\star}(s)\pi_{h}^{k}(a\mid s) \big(\hat\aggQ_{h}^{k}(s, a) - B_{h}^{k}(s, a) \big)^{2}
            \\
                &
                \leq \tilde O \bigg(\frac{H}{\eta}
                    +\eta H^{5}K\bigg) 
                        + 2\eta\!\! \sum_{k,h, s, a} \occ_{h}^{\star}(s)\pi_{h}^{k}(a\mid s)\hat\aggQ_{h}^{k}(s, a)^{2}
    \end{align*}
    In the inequality we've used the fact that $b_h^k(s) \leq 3 H$, and thus, $B_h^k(s,a) \leq 3H^2$ since it is a $Q$-function with respect to $b_h^k(s)$.
    Once again, the second sum here may not be bounded due to the mismatch between $\occ_{h}^{\star}(s)$ and $\occ_{h}^{k}(s)$ in the denominator of $\hat\aggQ_h^k(s,a)$. However, note that $\bbE_k[\hat\aggQ_h^k(s,a)^2] \leq H^2 / (\occ_h^k(s,a) + \gamma)$ and using standard concentration inequalities we can bound the the last term in the last display by,
    \begin{align*}
        & 2\eta H^{2}\sum_{k, h, s, a} \frac{\occ_{h}^{\star}(s)\pi_{h}^{k}(a\mid s)}{\occ_{h}^{k}(s, a) + \gamma} + \tilde{O} \l(\eta\frac{H^{3}}{\gamma^{2}} \r)
        = \underbrace{\frac{1}{3}\sum_{k, h, s}\occ_{h}^{\star}(s)b_{h}^{k}(s)}_{(ii)}
                        + \tilde{O} \l(\frac{H^{2}}{\gamma} \r)
        ,
    \end{align*}
    where we've set $\eta = \gamma / (2H)$ as in the statement of the theorem. Finally, recall that $B_h^k$ is the $Q$-function with respect to $b^k$ as losses. Applying the standard value difference lemma (\cref{lemma: value diff}) we have, 
    \begin{align*}
     \textsc{Bonus}
        &= \sum_k V^{\pi^{k}}_1(\sinit;b^k) - V^{\pi^{\star}}(\sinit;b^k )
          = \sum_{k,h,s}\occ_{h}^{k}(s)b_{h}^{k}(s) - \sum_{k,h,s}\occ_{h}^{\star}(s)b_{h}^{k}(s).
    \end{align*}
    The negative  term exactly cancels out $(i)$ and $(ii)$ from $\textsc{Bias}_1$ and $\textsc{Reg}$. The positive term is bounded by,
    \begin{align*}
        3 \gamma H \sum_{k=1}^{K}\sum_{h,s,a}\frac{\occ_{h}^{k}(s) \pi^{k}_h(a \mid s) }{\occ_{h}^{k}(s)\pi^{k}_h(a \mid s) + \gamma}
        \leq
        3 \gamma H^2 SA K. 
    \end{align*}
    Summing all the terms and setting $\eta$ and $\gamma$ as in the statement of the theorem completes the proof.
\end{proof}

\section{Unknown Dynamics}
The adaptation of our algorithm to the unknown dynamics case is relatively straightforward. Since we don't know the transition function, we can't compute $\occ^k$ and can't perform the Bellman backup in the computation of $B^k$. As standard in the literature, we employ Bernstein-style confidence sets for the transition function. Specifically, let $\calP^k = \{ \calP^k_h(s,a) \}_{s,a,h}$ such that $p'_h(\cdot \mid s,a) \in \calP^k_h(s,a)$ if and only if $p'_h(\cdot \mid s,a) \in \Delta_\calS$ and for every $s' \in \calS$:
    \begin{align*}
        & |p'_h(s'\mid s,a) - \bar p^k_h(s'\mid s,a)|
        \le 4 \sqrt{\frac{\bar p^k_h(s'\mid s,a) \log \frac{10 H S A K}{\delta}}{n^k_h(s,a) \vee 1}} + \frac{10 
        \log \frac{10 H S A K}{\delta}}{n^k_h(s,a) \vee 1},
    \end{align*}
    where $n_h^k(s,a,s')$ is the number of times we visited $s,a$ at time $h$ and transitioned to $s'$, $n_h^k(s,a) = \sum_{s'}n_h^k(s,a,s')$ and $\bar p_h^k(s'\mid s,a) = n_h^k(s,a,s') / n_h^k(s,a)$ is the empirical transition probability.
    Based on $\calP^k$, we replace $\occ_h^k(s)$ in the definition of $\hat \aggQ_h^k(s,a)$ with an upper confidence bound $\ol \occ^k_h(s) :=\max_{p' \in \calP^k} \occ^{\pi^k,p'}_h(s)$ (see \Cref{alg:tabular-unknown-p}). Here, the notation $\occ^{\pi^k,p'}_h(s)$ represents the occupancy measure with respect to the transition $p'$ (instead of $p$). The local bonus $b_h^k(s)$ is adapted correspondingly and is exactly as in \citet{luo2021policy} - see definition in \Cref{alg:tabular-unknown-p main}. Finally, the bonus $B^k$ is estimated optimistically based on the transition $p' \in \calP^k$ that maximizes it, as described in the update computation of $\hat B^k$ in \Cref{alg:tabular-unknown-p main}.

\begin{algorithm}[t]
    \caption{Policy Optimization with Aggregated Bandit Feedback and Unknown Transition Function}  
    \label{alg:tabular-unknown-p main}
    \begin{algorithmic}
        \STATE \textbf{Input:} state space $\calS$, action space $\calA$, horizon $H$, learning rate $\eta > 0$, exploration parameter $\gamma > 0$, confidence parameter $\delta > 0$
        

        \STATE \textbf{Initialization:} 
        Set $\pi_{h}^{1}(a \mid s) = {1}/{A}$ for every $(h,s,a)$
        
        \FOR{$k=1,2,\dots,K$}
            
            \STATE Play episode $k$ with policy $\pi^k$ and observe aggregated bandit feedback $L^k_{1:H} = \sum_{h=1}^H \ell_h^k(s^k_h,a^k_h) $
            

            
            \STATE $\ol \occ^k_h(s) =\max_{p' \in \calP^k} \occ^{\pi^k,p'}_h(s)$ 
            \STATE 
            $\ul \occ^k_h(s) = \min_{p' \in \calP^k} \occ^{\pi^k,p'}_h(s)$
            
            
            \STATE 
            $
                \hat\aggQ_{h}^{k}(s, a)  = \frac{L^{k}_{1:H}}{\ol \occ_{h}^{k}(s, a) + \gamma} \mathbb{I}_{h}^{k}(s, a)
            $
            \STATE {\color{gray} \# Bonus Computation}
            \STATE Set $\hat B^k_{H+1}(s,a) = 0$ for every $(s,a) \in \calS \times \calA$.
            
            \FOR{$h=H,H-1,\dots,1$}
                
                \FOR{$(s,a) \in \calS \times \calA$}
                
                    \STATE $\tilde{b}^k_h(s) = \sum_{a \in \calA} \frac{3 \gamma H \pi^k_h(a \mid s) }{\ol \occ^k_h(s) \pi^k_h(a \mid s) + \gamma}$
                    \STATE $\bar b_{h}^{k}(s) = \sum_{a\in\calA}\frac{H\pi_{h}^{k}(a \mid s)(\ol \occ^{k}_h(s,a) - \ul \occ^{k}_h(s,a) )}{\ol \occ_{h}^{k}(s)\pi^k_h(a\mid s) + \gamma}$
                    
                    \STATE $b_h^k(s) = \tilde b_h^k(s) + \bar b_h^{k}(s)$
                    \STATE 
                        $\hat B_{h}^{k}(s,a) =b_{h}^{k}(s)
                        + \max_{p' \in \calP^k_h(s,a)} \sum_{s' \in \calS,a' \in \calA} p'_h(s'\mid s,a) \pi_{h+1}^{k}(a' \mid s')
                        \cdot \hat B^k_{h+1}(s',a')    
                    $

                \ENDFOR
            \ENDFOR
            \STATE {\color{gray} \# Policy Improvement}
            \STATE For every $(s,a,h) \in \calS \times \calA \times [H]$ by:
            \[
                \pi^{k+1}_h(a \mid s) = \frac{\pi^k_h(a \mid s) e^{-\eta( \hat \aggQ_{h}^{k}(s,a) - \hat B_h^k(s,a) )} } 
                {\sum_{a'} \pi^k_h(a' \mid s) e^{-\eta( \hat \aggQ_{h}^{k}(s,a') -\hat  B_h^k(s,a') ) }}.
            \]
        \ENDFOR
    \end{algorithmic}
\end{algorithm}

The regret guarantee of our algorithm is established in the following theorem.

\begin{theorem}
\label{theorem:unknown}
    Under unknown dynamics,
    running \Cref{alg:tabular-unknown-p main} with
     $\eta=({H \sqrt{S A K} + H^2\sqrt{K}})^{-1}$ and $\gamma = 2\eta H$, guarantees with probability $1-\delta$,
    \begin{align*}
        R_K \leq  \tilde O (H^{3}S\sqrt{AK} 
                        + H^{4}S^{3}A).
    \end{align*}
\end{theorem}

The second term of the regret is typically of low order for sufficiently large $K$. The above regret improves upon \citet{cohen2021online} by a factor of $H^2 S^5 A^2$ and provides a high-probability regret bound instead of an expected one. \citet{cohen2021online} uses a reduction to distorted linear bandits and requires solving a convex optimization problem in each iteration. In contrast, our algorithm benefits from a more computationally efficient, closed-form update. Our bound also improves upon \citet{efroni2021reinforcement} and \citet{cassel2024near} for the case of tabular MDPs by a factor of $SA\sqrt{H}$ even though they consider stochastic i.i.d losses as opposed to adversarial losses in our setting. The gap from the lower bound is a factor $H \sqrt{S}$, but it is worth noting that even in the semi-bandit case, the best-known regret is $H^2 S \sqrt{A K}$ \citep{jin2019learning} which is achieved by a less efficient occupancy-measure-based algorithm. In fact, even though we consider less informative feedback, we match the state-of-the-art regret for PO with semi-bandit feedback \citep{luo2021policy}.

The proof of \Cref{theorem:unknown} follows similar steps as in the known dynamics case, but includes additional steps to control the transition estimation error. These  utilize standard techniques that we highlight in the following proof sketch.
\begin{proof}[sketch]
    We utilize the same regret decomposition as in \Cref{eq:regret decomposition} but replace $B^k$ with $\hat B^k$ in the $\texttt{Reg}$ and \texttt{Bonus} terms. The estimator $\hat U^k$ remains optimistic (i.e., in expectation it is smaller than $U^k$) since with high probability $\ol \occ_h^k(s) \geq \occ^k_h(s)$ for any $h, s,$ and $k$.  Thus, $\textsc{Bias}_2 \leq \tilde{O}(H^2 / \gamma)$ with high probability in a similar way to the known dynamics case. 
    In $\texttt{Bias}_1$, on the other hand, using $\ol \occ_h^k(s)$ instead of $\occ_h^k(s)$ introduces additional bias of 
    order of,
    \begin{align*}
                           & \sum_{k,h, s, a}\occ_{h}^{\star}(s)
                                \pi_{h}^{k}(a\mid s) H 
                                    \frac{\ol \occ_{h}^{k}(s, a) - \occ_{h}^{k}(s, a)}{\ol \occ_{h}^{k}(s, a) + \gamma}
                            \leq
                                                        \sum_{k,h, s, a}\occ_{h}^{\star}(s)
                                \pi_{h}^{k}(a\mid s) H 
                                    \frac{\ol \occ_{h}^{k}(s, a) - \ul \occ_{h}^{k}(s, a)}{\ol \occ_{h}^{k}(s, a) + \gamma},
    \end{align*}
    where the inequality follows from the fact that $\ul \occ_{h}^{k}(s, a) \leq \occ_{h}^{k}(s, a)$ w.h.p. Note that this is exactly $\sum_{k,h, s, a}\occ_{h}^{\star}(s)
                            \bar{b}_{h}^{k}(s)$.
    This term is in addition to the terms we already had in $\texttt{Bias}_1$ in the known dynamics case. In total we have,
    \[
	\textsc{Bias}_{1} \leq \frac{2}{3} \sum_{k,h,s}
	                  \occ_{h}^{\star}(s) \tilde b_{h}^{k}(s)
	                  +  \sum_{k,h,s}\occ_{h}^{\star}(s) \bar b_{h}^{k}(s) + \tilde O\l(
	                  \frac{H^{2}}{\gamma}\r )                                              
    \]

\texttt{Reg} is bounded in a similar way to the known dynamics case. Given that the estimator is optimistic and exhibits lower variance when using the upper confidence bound on the occupancy measure, we get that \texttt{Reg} is bounded by,
\begin{align*}
    \frac{H\ln A}{\eta} 
                + 16\eta H^{5}K 
                    + \frac{1}{3}\sum_{k, h, s}\occ_{h}^{\star}(s) \tilde b_{h}^{k}(s) 
                        + \tilde O \l(\frac{H^{2}}{\gamma} \r).
\end{align*}
Finally, note $\hat B^k$ is not an exact $Q$-function since we don't know the actual transition function. Thus, we can not directly use the value difference lemma to show that $\texttt{Bonus} = \sum_{k,h,s}\occ_{h}^{k}(s)b_{h}^{k}(s) - \sum_{k,h,s}\occ_{h}^{\star}(s)b_{h}^{k}(s)$. However, using the fact that we update $\hat B$ with a transition function in the confidence set that maximizes it, we are able to show that w.h.p,
\begin{align*}
            \textsc{Bonus}  \leq  \sum_{k,h,s}\ol\occ_{h}^{k}(s)b_{h}^{k}(s)
                    - \sum_{k,h,s}\occ_{h}^{\star}(s)b_{h}^{k}(s).
\end{align*}
Once again, the negative term above cancels the terms that depend on $b^k$ in $\texttt{Bias}_1$ and $\texttt{Reg}$. Recall that $b_h^k(s) = \tilde b_h^k(s) + \bar b_h^k(s)$. Exactly as in the known dynamics case, $\sum_{k,h,s}\ol\occ_{h}^{k}(s) \tilde b_{h}^{k}(s) \leq O(\gamma H^2 S A K)$. The term $\sum_{k,h,s}\ol\occ_{h}^{k}(s) \bar b_{h}^{k}(s)$ equals to,
\begin{align*}
    &  H\sum_{k, h, s}\ol\occ_{h}^{k}(s, a)
                                    \frac{\ol \occ_{h}^{k}(s, a) - \ul \occ_{h}^{k}(s, a)}{\ol \occ_{h}^{k}(s, a) + \gamma}
                                \leq  H\sum_{k, h, s}
                                    |\ol \occ_{h}^{k}(s, a) - \ul \occ_{h}^{k}(s, a)|.
\end{align*}
The last is a standard transition estimation error and is bounded by $\tilde O ( H^3 S\sqrt{ A K } + H^3 S^{3} A)$ \citep{jin2019learning}. 
Summing all the terms and setting $\eta$ and $\gamma$ as in the statement of the theorem completes the proof.
\end{proof}

\section{Lower Bound}
Our lower bound uses a lower bound for the multi-task bandit problem (see for example \citet{lattimore2020bandit}). In the multitask bandit problem, the learner
faces simultaneously $H$ instances of the $A$-armed bandit problem. At each round $k\in[K]$, the learner selects $H$ actions, one for each bandit problem, and observes the sum of the losses associated with these $H$ actions. This scenario can be seen as analogous to MDPs with a single state (i.e., $S=1$) and aggregate bandit feedback. This is due to the fact that whenever we have only a single state no information is gained within the episode. (Recall that the losses are horizon-dependent, so we have $\ell_h(s_0,a)$, for each action $a$, time $h\in[H]$ and using the single state $s_0$.)

\begin{lemma}[Theorem 1 in \citet{cohen2017tight}]
    \label{lemma:lower bound multitask bandits}
    Assume that $A\geq 2$. Any learning algorithm for the multi-task bandit problem must incur at least $\tilde \Omega(H^2\sqrt{AK})$ expected regret in the worst case.
\end{lemma}

With that in hand, we obtain the following for online MDPs with aggregate bandit feedback.

\begin{theorem}
\label{thm:lower bound}
    Assume that $H,S,A\geq 2$ and $K \geq 2 S$. Any learning algorithm for the online MDPs with known dynamics and aggregate bandit feedback problem must incur at least $\Omega(H^2\sqrt{S A K})$ expected regret in the worst case.
\end{theorem}
The above lower bound shows that our regret upper bound for the known dynamics case is tight up to poly-logarithmic factors whenever $\sqrt{SA} \geq H$. 
For unknown dynamics, clearly the same lower bound holds, and we have a multiplicative gap of $H\sqrt{S}$.
We note that determining the exact optimal bound for the unknown dynamics case remains an open problem, both with aggregate bandit feedback as well as with the well-studied semi-bandit feedback.

Due to space limitations, the proof is deferred to \Cref{sec: appendix lower bound}. At a high level, the proof is constructed as follows: Consider an MDP where in the first step of each episode, the agent transitions to each state with an equal probability of $1/S$, regardless of the chosen action. From the second step onward, the agent remains in the same state for the remainder of the episode, wherein a hard multi-task bandit problem is encoded in each state. Roughly speaking, each state is visited approximately $K / S$ times. Using \Cref{lemma:lower bound multitask bandits}, the regret from visits to each state is approximately $H^2 \sqrt{A K / S}$. Summing the regret across all states gives a lower bound of $H^2 \sqrt{S A K}$.

\section{Discussion and Future Work}
In this paper, we introduced the concept of $U$-functions, which allows us to establish the first regret bounds for online MDPs with aggregate bandit feedback using PO algorithms. One of the advantages of the PO framework is its natural extension to function approximation \citep{luo2021policy, sherman2023improved, dai2023refined, liu2023towards}. It would be interesting to see whether the $U$-function concept could be useful in achieving regret bounds with aggregate bandit feedback for environments with infinitely many states under a function approximation assumption. 
We note that the main challenge in extending our results to linear MDPs, for example, is that the $U$-function is not linear under this assumption. 
Still, it is possible that in such settings, the $U$-function would have certain properties that would allow achieving sub-linear regret.

\newpage
\section*{Acknowledgments}
This project has received funding from the European Research Council (ERC) under the European Union’s Horizon 2020 research and innovation program (grant agreement No. 882396), by the Israel Science Foundation, the Yandex Initiative for Machine Learning at Tel Aviv University and a grant from the Tel Aviv University Center for AI and Data Science (TAD).


\newpage

\bibliography{main}
\bibliographystyle{abbrvnat}

\newpage
\appendix
\onecolumn

\section{Summery of notations}
For convenience, the table below summarizes most of the notation used throughout the paper.
\begin{table}[ht]
    \centering
    \begin{tabular}{c|l}
    $H$ & Horizon length of each episode \\
    $K$ & Total number of episodes \\
    $S$ & Number of states in the MDP \\
    $A$ & Number of actions in the MDP \\
    $p$ & Transition function of the MDP \\
    $\ell^k$ & Loss function in episode $k$\\
    $R_K$ & Cumulative regret over $K$ episodes \\
    $V_h^\pi(s;\ell)$ & Value function at state $s$ and time $h$ under policy $\pi$ and loss function $\ell$ \\
    $V^k_h(s)$ & Value under policy $\pi^k$ and loss function $\ell^k$. I.e., $V^k_h(s) = V_h^{\pi^k}(s;\ell^k)$ \\
    $Q_h^\pi(s,a;\ell)$ & $Q$-function at state $s$, action $a$, and time $h$ under policy $\pi$ and loss function $\ell$ \\
    $Q^k_h(s,a)$ & $Q$-function under policy \(\pi^k\) and loss function \(\ell^k\). I.e., \(Q^k_h(s,a) = Q_h^{\pi^k}(s,a;\ell^k)\) \\
    $\aggQ_h^\pi(s,a;\ell)$ & \makecell[l]{Expected total loss of the episode, conditioned on taking action $a$ in state $s$ at time $h$\\ under policy $\pi$} \\
    $\aggQ^k_h(s,a)$ & Conditional expected total loss with respect to $\pi^k$ and $\ell^k$. I.e., $\aggQ^k_h(s,a) = \aggQ_h^{\pi^k}(s,a;\ell^k)$ \\
    $\vp_{h}^\pi(s;\ell)$ & \makecell[l]{Expected cumulative loss up to time $h-1$, conditioned on reaching state $s$ at time $h$ \\under policy $\pi$} \\
    $\vp_{h}^{k}(s)$ & Conditional expected loss with respect to $\pi^k$ and $\ell^k$. I.e., $\vp_{h}^{k}(s) = \vp_{h}^{\pi^k}(s;\ell^k)$ \\
    $\occ_h^\pi(s,a)$ & \makecell[l]{Occupancy measure: probability of being in state $s$ and taking action $a$ at time $h$ \\ under policy $\pi$} \\
    $\occ_h^k(s,a)$ & Occupancy measure under $\pi^k$: $\occ_h^k(s,a) =\occ_h^{\pi^k}(s,a) $ \\
    $\occ^{\pi,p'}_h(s)$ & Occupancy measure with respect to a transition function $p'$\\
    $\bbE_k[\cdot]$ & Expectation condition on the policy $\pi^k$: $\bbE_k[\cdot] = \bbE[\cdot \mid \pi^k]$\\
    $\bbI_h^k(s,a)$ & The indicator for visiting $s$ and taking $a$ at time $h$ in episode $k$: $\bbI_h^k(s,a) = \bbI\{s_h^k = s,a_h^k = a\}$ \\
    $\hat \aggQ^k_h(s,a)$ & Estimator of $\aggQ^k_h(s,a)$ - See \Cref{alg:tabular-known-p main,alg:tabular-unknown-p main} \\
    $b_h^k(s)$ & Intermediate bonus function - See \Cref{alg:tabular-known-p main,alg:tabular-unknown-p main}\\
    $B_h^k(s,a)$ & Bonus function, defined as the $Q$-function with respect to $b^k$ - See \Cref{alg:tabular-known-p main}\\
    $\hat B_h^k(s,a)$ & Estimator of $B_h^k(s,a)$ - See \Cref{alg:tabular-unknown-p main}\\
    $\calP^k$ & Confidence set of transition functions at episode $k$\\
    $\ol{\occ}^k_h(s,a)$ & Upper confidence bound on the occupancy measure - See \Cref{alg:tabular-unknown-p main}\\
    $\ul{\occ}^k_h(s,a)$ & Lower confidence bound on the occupancy measure - See \Cref{alg:tabular-unknown-p main}\\
    $\logterm$ & A logarithmic factor of $\logtermval$
    \end{tabular}

    \label{tab:notation}
\end{table}

\newpage
\section{Known Dynamics}

\begin{theorem}[Restatement of \Cref{theorem:known}]
    Under known dynamics,
    running \Cref{alg:tabular-known-p main} with $\eta=\frac{\sqrt{\logterm}}{H \sqrt{S A K} + H^2\sqrt{K}}$ and $\gamma = 2\eta H$ for $\logterm = \logtermval$, guarantees with probability $1-\delta$,
    \begin{align*}
        R_K \lesssim H^{2}\sqrt{SAK\logterm}+H^{3}\sqrt{K\logterm}.
    \end{align*}
\end{theorem}

\begin{proof}
Using \Cref{eq: regret U}, we can break the regret of the algorithm as,
\begin{align*}
    R_{K} &  = \sum_{k = 1}^{K} \sum_{h, s}\occ_{h}^{\star}(s) \l\langle \pi_{h}^{k}(\cdot\mid s) - \pi_{h}^{\star}(\cdot\mid s), U_{h}^{k}(s, \cdot) \r\rangle
        \\
        &  
        = \underbrace{\sum_{k = 1}^{K} \sum_{h, s}\occ_{h}^{\star}(s) \l\langle \pi_{h}^{k}(\cdot\mid s), \aggQ_{h}^{k}(s, \cdot) - \hat\aggQ_{h}^{k}(s, \cdot) \r\rangle }_{\textsc{Bias}_{1}}
        + \underbrace{\sum_{k = 1}^{K} \sum_{h, s}\occ_{h}^{\star}(s) \l\langle \pi_{h}^{\star}(\cdot\mid s), \hat\aggQ_{h}^{k}(s, \cdot) - \aggQ_{h}^{k}(s, \cdot) \r\rangle }_{\textsc{Bias}_{2}}
        \\
        & \quad + \underbrace{\sum_{k = 1}^{K} \sum_{h, s}\occ_{h}^{\star}(s) \l\langle \pi_{h}^{k}(\cdot\mid s) - \pi_{h}^{\star}(\cdot\mid s), \hat\aggQ_{h}^{k}(s, \cdot) - B_{h}^{k}(s, \cdot) \r\rangle }_{\textsc{Reg}}
        \\
        & \qquad 
        + \underbrace{\sum_{k = 1}^{K} \sum_{h, s}\occ_{h}^{\star}(s) \l\langle \pi_{h}^{k}(\cdot\mid s) - \pi_{h}^{\star}(\cdot\mid s), B_{h}^{k}(s, \cdot) \r\rangle }_{\textsc{Bonus}},
\end{align*}
Due to the optimism of the estimator $\textsc{Bias}_{2}\leq\tilde{O}(H^{2}/\gamma)$
under the good event $G_3$ which holds with high probability (see \cref{lemma:good bias2}).
In \cref{lemma:bias1,lemma:reg,lemma:bonus} we bound $\textsc{Bias}_1$, $\textsc{Reg}$ and $\textsc{Bonus}$. Overall we get that,
\begin{align*}
    R_{K} & \lesssim\frac{H\ln A}{\eta} + \eta H^{5}K + \gamma H^{2}SAK + \frac{H^{2}\logterm}{\gamma}
\end{align*}
Plugging $\eta$ and $\gamma$ completes the proof.

\end{proof}

\subsection{Good event}
\label{section: good even known}

We define the following good event $G = \bigcap_{i=1}^3 G_i$ which holds with high probability (\cref{lemma:good event}):

\begin{align*}
    G_1 &= \l\{ \sum_{k=1}^K \sum_{h, s, a}\occ_{h}^{\star}(s)\pi_{h}^{k}(a\mid s) \l(\bbE_{k} \l[\hat\aggQ_{h}^{k}(s, a) \r] - \hat\aggQ_{h}^{k}(s, a) \r) 
        \leq \frac{1}{3}\sum_{k = 1}^{K} \sum_{h, s}\occ_{h}^{\star}(s)b_{h}^{k}(s) + \frac{H^{2}\log\frac{6}{\delta}}{\gamma} 
            \r\}
\\
    G_2 &= \l\{\sum_{k=1}^K \sum_{h, s, a} \frac{\occ_{h}^{\star}(s)\pi_{h}^{k}(a\mid s)}{\occ_{h}^{k}(s, a) + \gamma} \l(\frac{ \mathbb{I}\{s_{h}^{k} = s, a_{h}^{k} = a\}}{\occ_{h}^{k}(s, a) + \gamma} - 1 \r) 
        \leq   \frac{H\ln\frac{6H}{\delta}}{2 \gamma^2} 
        \r \}
\\
    G_3 &= \l\{
        \sum_{k = 1}^{K} \sum_{h, s}\occ_{h}^{\star}(s) \l\langle \pi_{h}^{\star}(\cdot\mid s), \hat\aggQ_{h}^{k}(s, \cdot) - \aggQ_{h}^{k}(s, \cdot) \r\rangle \leq \frac{H^2}{2\gamma}\ln\frac{6H}{\delta}
            \r \}
\end{align*}

Under the good event the regret will be deterministically bounded.

\begin{lemma}[Event $G_1$]
    \label{lemma:good bias1}
    With probability $1-\delta$,
    \begin{align*}
        \sum_{k=1}^K \sum_{h, s, a}\occ_{h}^{\star}(s)\pi_{h}^{k}(a\mid s) \l(\bbE_{k} \l[\hat\aggQ_{h}^{k}(s, a) \r] - \hat\aggQ_{h}^{k}(s, a) \r) 
            \leq \frac{1}{3}\sum_{k = 1}^{K} \sum_{h, s}\occ_{h}^{\star}(s)b_{h}^{k}(s) + \frac{H^{2}\log\frac{1}{\delta}}{\gamma}
    \end{align*}

\end{lemma}

\begin{proof}
    Let $Y_{k}=\sum_{h,s} \occ_{h}^{\star}(s) \l\langle \pi_{h}^{k}(\cdot\mid s),\hat\aggQ_{h}^{k}(s,\cdot) \r\rangle $.
    To bound $\sum_{k}\mathbb{E}_{k}[Y_{k}]-\sum_{k}Y_{k}$ we'll use
    a form of Freedman's Inequality (\cref{lemma:freedman}). For that we need to bound $\mathbb{E}_{k}[Y_{k}^{2}]$:
    \begin{align*}
        \mathbb{E}_{k}[Y_{k}^{2}] &  =  \mathbb{E}_{k} \l[ \l(\sum_{h, s, a}\occ_{h}^{\star}(s)\pi_{h}^{k}(a\mid s)\hat\aggQ_{h}^{k}(s, a) \r)^{2} \r]
            \\
            &  \leq  \mathbb{E}_{k} \l[ \l(\sum_{h, s, a}\occ_{h}^{\star}(s)\pi_{h}^{k}(a\mid s) \r) \l(\sum_{h, s, a}\occ_{h}^{\star}(s)\pi_{h}^{k}(a\mid s) \l(\hat\aggQ_{h}^{k}(s, a) \r)^{2} \r) \r] \tag{Cauchy–Schwarz inequality}
            \\
            &  =  H \mathbb{E}_{k} \l[\sum_{h, s, a}\occ_{h}^{\star}(s)\pi_{h}^{k}(a\mid s) \l(\hat\aggQ_{h}^{k}(s, a) \r)^{2} \r]
            \\
            &  =  H \mathbb{E}_{k} \l[\sum_{h, s, a}\occ_{h}^{\star}(s)\pi_{h}^{k}(a\mid s)\frac{(L^{k}_{1:H})^{2}}{(\occ_{h}^{k}(s, a) + \gamma)^{2}} \mathbb{I}_{h}^{k}(s, a) \r]
            \\
            &  \leq  H^{3} \mathbb{E}_{k} \l[\sum_{h, s, a}\occ_{h}^{\star}(s)\pi_{h}^{k}(a\mid s)\frac{ \mathbb{I}_{h}^{k}(s, a)}{(\occ_{h}^{k}(s, a) + \gamma)^{2}} \r] \tag{$L_{1:H}^k \leq H$}
            \\
            &  =  H^{3}\sum_{h, s, a}\occ_{h}^{\star}(s)\pi_{h}^{k}(a\mid s)\frac{\occ_{h}^{k}(s, a)}{(\occ_{h}^{k}(s, a) + \gamma)^{2}} \tag{$\bbE_k[\bbI_h^k(s,a)] = \occ_h^k(s,a)$}
            \\
            &  \leq  H^{3}\sum_{h, s, a}\occ_{h}^{\star}(s)\frac{\pi_{h}^{k}(a\mid s)}{\occ_{h}^{k}(s, a) + \gamma}.
    \end{align*}
    By \cref{lemma:freedman} with probability $1-\delta$,
    \[
    \sum_{k} \mathbb{E}_{k}[Y_{k}] - \sum_{k}Y_{k} \leq \alpha\sum_{k} \mathbb{E}_{k}[Y_{k}^{2}] + \frac{\log\frac{1}{\delta}}{\alpha}
    \]
    where $\alpha\in(0,1/R]$ for $R$ such that $|Y_{k}|\leq R$ for
    any $k$. In our case, $|Y_{k}|\leq H^{2}/\gamma$. And so, 
    \[
        \sum_{k} \mathbb{E}_{k}[Y_{k}] - \sum_{k}Y_{k} \leq \frac{1}{3}\sum_{k = 1}^{K} \sum_{h, s} \occ_{h}^{\star}(s)b_{h}^{k}(s) + \frac{H^{2}\log\frac{1}{\delta}}{\gamma}.
    \]
\end{proof}

\begin{lemma}[Event $G_2$]
    \label{lemma:good reg}
    With probability $1-\delta$,
    \begin{align*}
        \sum_{k, h, s, a} \frac{\occ_{h}^{\star}(s)\pi_{h}^{k}(a\mid s)}{\occ_{h}^{k}(s, a) + \gamma} \l(\frac{ \mathbb{I}\{s_{h}^{k} = s, a_{h}^{k} = a\}}{\occ_{h}^{k}(s, a) + \gamma} - 1 \r) 
            \leq   \frac{H\ln\frac{H}{\delta}}{2 \gamma^2} 
    \end{align*}
\end{lemma}

\begin{proof}
    Follows directly from \cref{lemma:concentration} with with $Z_{h}^{k}(s,a) = z_{h}^{k}(s,a)=\frac{\occ_{h}^{\star}(s)\pi_{h}^{k}(a\mid s)}{\occ_{h}^{k}(s,a) + \gamma} \leq \frac 1 \gamma$
and $\tilde{\occ}_{h}^{k}(s,a)=\occ_{h}^{k}(s,a)$.
\end{proof}

\begin{lemma}[Event $G_3$]
    \label{lemma:good bias2}
    With probability $1-\delta$,
    \begin{align*}
        \sum_{k = 1}^{K} \sum_{h, s}\occ_{h}^{\star}(s) \l\langle \pi_{h}^{\star}(\cdot\mid s), \hat\aggQ_{h}^{k}(s, \cdot) - \aggQ_{h}^{k}(s, \cdot) \r\rangle \leq \frac{H^2}{2\gamma}\ln\frac{H}{\delta}
    \end{align*}
\end{lemma}

\begin{proof}
    By invoking \cref{lemma:concentration} with  $Z_{h}^{k}(s,a)
        =\occ_{h}^{\star}(s)\pi_{h}^{k}(a,s)[\mathbb{I}_{h}^{k}(s,a)L^{k}_{1:H} + (1-\mathbb{I}_{h}^{k}(s,a))\aggQ_{h}^{k}(s,a)]
            \leq H$,
    $z_{h}^{k}(s,a)
        =\mathbb{E}_{k}[Z_{h}^{k}(s,a)]
            =\occ_{h}^{\star}(s)\pi_{h}^{k}(a,s)\aggQ_{h}^{k}(s,a)$
    and $\tilde{\occ}_{h}^{k}(s,a) = \occ_{h}^{k}(s,a)$, we get,
    \begin{align*}
        & \sum_{k = 1}^{K} \sum_{h, s}\occ_{h}^{\star}(s) \l\langle \pi_{h}^{\star}(\cdot\mid s), \hat\aggQ_{h}^{k}(s, \cdot) - \aggQ_{h}^{k}(s, \cdot) \r\rangle  
        \\
        &\quad =  \sum_{k = 1}^{K} \sum_{h, s, a} \frac{\occ_{h}^{\star}(s)\pi_{h}^{k}(a, s) \mathbb{I}_{h}^{k}(s, a)L^{k}_{1:H}}{\occ_{h}^{k}(s, a) + \gamma} - \sum_{k = 1}^{K} \sum_{h, s, a}\occ_{h}^{\star}(s)\pi_{h}^{k}(a, s)\aggQ_{h}^{k}(s, a) 
                \leq \frac{H^2}{2\gamma}\ln\frac{H}{\delta}
    \end{align*}
\end{proof}

\begin{lemma}
    \label{lemma:good event}
    Under \Cref{alg:tabular-known-p main},
    the good event $G = \bigcap_{i=1}^3 G_i$ holds with probability of at least $1-\delta$.
\end{lemma}
\begin{proof}
    The proof directly follows from invoking \Cref{lemma:good reg,lemma:good bias1,lemma:good bias2} with $\delta / 3$ and taking the union bound.
\end{proof}

\subsection{Bound on $\textsc{Bias}_{1}$}

\begin{lemma}
    \label{lemma:bias1}
    Under the good event $G_1$,
    \begin{align*}
        \textsc{Bias}_{1} 
            \leq \frac{2}{3}\sum_{k = 1}^{K} \sum_{h, s} \occ_{h}^{\star}(s)b_{h}^{k}(s) + \frac{H^{2}\log\frac{6}{\delta}}{\gamma}.
    \end{align*}
\end{lemma}

Let $Y_{k}=\sum_{h,s} \occ_{h}^{\star}(s) \l\langle \pi_{h}^{k}(\cdot\mid s),\hat\aggQ_{h}^{k}(s,\cdot) \r\rangle $
\begin{align*}
    \textsc{Bias}_{1} 
        &  = \sum_{k = 1}^{K} \sum_{h, s}\occ_{h}^{\star}(s) \l\langle \pi_{h}^{k}(\cdot\mid s), \aggQ_{h}^{k}(s, \cdot) \r\rangle  - \sum_{k}\mathbb{E}_{k}[Y_{k}] + \sum_{k}\mathbb{E}_{k}[Y_{k}] - \sum_{k}Y_{k}
\end{align*}

Under the good event $G_1$,
    $
        \sum_{k} \mathbb{E}_{k}[Y_{k}] - \sum_{k}Y_{k} 
            \leq \frac{1}{3}\sum_{k = 1}^{K} \sum_{h, s} \occ_{h}^{\star}(s)b_{h}^{k}(s) + \frac{H^{2}\log\frac{6}{\delta}}{\gamma}.
    $
Using \cref{lemma:unbiased},
\begin{align*}
    \sum_{h, s}\occ_{h}^{\star}(s) \l\langle \pi_{h}^{k}(\cdot\mid s), \aggQ_{h}^{k}(s, \cdot) \r\rangle  -  \mathbb{E}_{k}[Y_{k}] 
        &  = \sum_{k = 1}^{K} \sum_{h, s, a}\occ_{h}^{\star}(s)\pi_{h}^{k}(a\mid s) \l(\aggQ_{h}^{k}(s, a) -  \mathbb{E}_{k} \l[\hat\aggQ_{h}^{k}(s, a) \r] \r)
        \\
            &  = \sum_{k = 1}^{K} \sum_{h, s, a}\occ_{h}^{\star}(s)\pi_{h}^{k}(a\mid s)\aggQ_{h}^{k}(s, a) \l(1 - \frac{\occ_{h}^{k}(s, a)}{\occ_{h}^{k}(s, a) + \gamma} \r)
            \\
            & \leq H\sum_{k = 1}^{K} \sum_{h, s, a}\occ_{h}^{\star}(s)\pi_{h}^{k}(a\mid s) \l(1 - \frac{\occ_{h}^{k}(s, a)}{\occ_{h}^{k}(s, a) + \gamma} \r)
            \qquad\qquad
            \tag{$U_h^k(s,a) \leq H$}
            \\
            & =\sum_{k = 1}^{K} \sum_{h, s, a}\occ_{h}^{\star}(s)\frac{\gamma H\pi_{h}^{k}(a\mid s)}{\occ_{h}^{k}(s, a) + \gamma}
            \\
            &  = \frac{1}{3}\sum_{k = 1}^{K} \sum_{h, s}\occ_{h}^{\star}(s)b_{h}^{k}(s).
\end{align*}

\subsection{Bound on \textsc{Reg}}
\begin{lemma}
    \label{lemma:reg}
    For $\eta \leq \frac{\gamma}{2H}$, under the good $G_2$, 
    \begin{align*}
        \textsc{Reg} 
            &  \leq \frac{H\ln A}{\eta} 
                + 9\eta H^{5}K 
                    + \frac{1}{3}\sum_{k, h, s}\occ_{h}^{\star}(s)b_{h}^{k}(s) 
                        + O \l(\frac{H^{2}\logterm}{\gamma} \r)
    \end{align*}
\end{lemma}
\begin{proof}
    Using standard entropy regularized OMD guarantee (\cref{lemma:OMD}),
    \begin{align*}
    \nonumber
        \textsc{Reg} 
            & \leq \frac{H\ln A}{\eta} + 2\eta\sum_{k = 1}^{K} \sum_{h, s, a}\occ_{h}^{\star}(s)\pi_{h}^{k}(a\mid s) \l(\hat\aggQ_{h}^{k}(s, a) - B_{h}^{k}(s, a) \r)^{2}
            \\
                & \leq \frac{H\ln A}{\eta} + 2\eta\sum_{k = 1}^{K} \sum_{h, s, a} \occ_{h}^{\star}(s)\pi_{h}^{k}(a\mid s)\hat\aggQ_{h}^{k}(s, a)^{2} + 9\eta H^{5}K,
    \end{align*}
    where the second inequality is since $b_h^k(s) \leq 3H$ and thus $B_h^k(s,a) \leq 3H^2$.
    For the middle term,
    \begin{align*}
        2\eta\sum_{k, h, s, a}\occ_{h}^{\star}(s)\pi_{h}^{k}(a\mid s)\hat\aggQ_{h}^{k}(s, a)^{2} &  \leq 2\eta\sum_{k, h, s, a}\occ_{h}^{\star}(s)\pi_{h}^{k}(a\mid s)\frac{H^{2} \mathbb{I}\{s_{h}^{k} = s, a_{h}^{k} = a\}}{(\occ_{h}^{k}(s, a) + \gamma)^{2}}
        \qquad\quad
        \tag{$L_{1:H}^k \leq H$}
        \\
        &  \leq 2\eta H^{2}\sum_{k, h, s, a} \frac{\occ_{h}^{\star}(s)\pi_{h}^{k}(a\mid s)}{\occ_{h}^{k}(s, a) + \gamma} + O \l(\eta\frac{H^{3}\logterm}{\gamma^{2}} \r) \tag{Under event $G_2$}
        \\
        &  = \frac{2\eta}{3\gamma}H\sum_{k, h, s}\occ_{h}^{\star}(s)b_{h}^{k}(s) + O \l(\eta\frac{H^{3}\logterm}{\gamma^{2}} \r)
        \\
        &  \leq \frac{1}{3}\sum_{k, h, s}\occ_{h}^{\star}(s)b_{h}^{k}(s) + O \l( \frac{H^{2}\logterm}{\gamma} \r).
    \end{align*}
    The last inequality is since $\eta \leq  \frac{H}{2 \gamma}$ as in the statement of the lemma.
\end{proof}

\subsection{Bound on $\textsc{Bonus}$}
\begin{lemma}
    \label{lemma:bonus}
    It holds that,
    $$
        \textsc{Bonus}  \leq 3 \gamma H^2 SA K - \sum_{k,h,s}\occ_{h}^{\star}(s)b_{h}^{k}(s).
    $$
\end{lemma}
\begin{proof}
    Recall that that $B_{h}^{k}$ is the $Q$-function of policy $\pi^k$ with respect to the cost function $b^k$. Hence, by the value difference
    difference lemma (\cref{lemma: value diff}),
    \begin{align*}
    \textsc{Bonus} = \sum_{k = 1}^{K} \sum_{h, s}\occ_{h}^{\star}(s) \l\langle \pi_{h}^{k}(\cdot\mid s) - \pi_{h}^{\star}(\cdot\mid s), B_{h}^{k}(s, \cdot) \r\rangle
        & = \sum_k V^{\pi^{k}}_1(\sinit;b^k) - V^{\pi^{\star}}(\sinit;b^k )\\
            &\qquad= \sum_{k,h,s}\occ_{h}^{k}(s)b_{h}^{k}(s) - \sum_{k,h,s}\occ_{h}^{\star}(s)b_{h}^{k}(s).
    \end{align*}

    For last,
    \begin{align*}
        \sum_{k=1}^{K}\sum_{h,s}\occ_{h}^{k}(s)b^k_h(s)
        & =
        3 \gamma H \sum_{k=1}^{K}\sum_{h,s,a}\frac{\occ_{h}^{k}(s) \pi^{k}_h(a \mid s) }{\occ_{h}^{k}(s)\pi^{k}_h(a \mid s) + \gamma}
        \leq
        3 \gamma H^2 SA K. 
    \end{align*}
\end{proof}

\newpage
\section{Unknown Dynamics}
\begin{algorithm}
    \caption{Policy Optimization with Aggregated Bandit Feedback and Unknown Transition Function \\(detailed version of \Cref{alg:tabular-unknown-p main})}  
    \label{alg:tabular-unknown-p}
    \begin{algorithmic}
        \STATE \textbf{Input:} state space $\calS$, action space $\calA$, horizon $H$, learning rate $\eta > 0$, exploration parameter $\gamma > 0$, confidence parameter $\delta > 0$.
        
        \STATE \textbf{Initialization:} 
        Set $\pi_{h}^{1}(a \mid s) = \frac{1}{A}$ and visit counters $n_h^1(s,a,s') = 0$, $n_h^1(s,a) = 0$ for every $(h,s,a,s') \in [H]\times \calS \times \calA \times \calS$.
        
        \FOR{$k=1,2,\dots,K$}
            
            \STATE Play episode $k$ with policy $\pi^k$ and observe aggregated bandit feedback $L^k = \sum_{h=1}^H \ell_h^k(s^k_h,a^k_h) $.
            
            \STATE Update visit counters for every $(h,s,a,s') \in [H] \times \calS \times \calA \times \calS$:
            $$n^{k}_h(s,a,s') = n^{k-1}_h(s,a,s') +  \bbI_h^{k-1} (s,a,s') \, ; \qquad n^{k}_h(s,a) = n^{k-1}_h(s,a) + \bbI_h^{k-1} (s,a).$$
            
            \STATE Compute empirical transition function $\bar p^{k}_h(s' \mid s,a) = \frac{n^k_h(s,a,s')}{\max \{ n^k_h(s,a), 1 \}}$ and confidence set $\calP^k = \{ \calP^k_h(s,a) \}_{s,a,h}$ such that $p'_h(\cdot \mid s,a) \in \calP^k_h(s,a)$ if and only if $\sum_{s'} p'_h(s'\mid s,a) = 1$ and for every $s' \in \calS$:
            \[
                |p'_h(s'\mid s,a) - \bar p^k_h(s'\mid s,a)| \le 4 \sqrt{\frac{\bar p^k_h(s'\mid s,a) \log \frac{10 H S A K}{\delta}}{n^k_h(s,a) \vee 1}} + \frac{10 
                \log \frac{10 H S A K}{\delta}}{n^k_h(s,a) \vee 1}.
            \]
            
            \STATE Compute occupancy measures $\ol \occ^k_h(s) =\max_{p' \in \calP^k} \occ^{\pi^k,p'}_h(s)$ and $\ul \occ^k_h(s) = \min_{p' \in \calP^k} \occ^{\pi^k,p'}_h(s)$.
            
            \STATE {\color{gray} \# Policy Evaluation}
            
            \STATE 
            $
                \hat\aggQ_{h}^{k}(s, a)  = \frac{L^{k}_{1:H}}{\ol \occ_{h}^{k}(s, a) + \gamma} \mathbb{I}_{h}^{k}(s, a)
            $
            \STATE {\color{gray} \# Bonus computation}
            \STATE Set $\hat B^k_{H+1}(s,a) = 0$ for every $(s,a) \in \calS \times \calA$.
            
            \FOR{$h=H,H-1,\dots,1$}
                
                \FOR{$(s,a) \in \calS \times \calA$}
                
                    \STATE $\tilde{b}^k_h(s) = \sum_{a \in \calA} \frac{3 \gamma H \pi^k_h(a \mid s) }{\ol \occ^k_h(s) \pi^k_h(a \mid s) + \gamma}$; \quad
                    $\bar b_{h}^{k}(s) = \sum_{a\in\calA}\frac{H\pi_{h}^{k}(a \mid s)(\ol \occ^{k}_h(s,a) - \ul \occ^{k}_h(s,a) )}{\ol \occ_{h}^{k}(s)\pi^k_h(a\mid s) + \gamma}$.
                    
                    \STATE $b_h^k(s) = \tilde b_h^k(s) + \bar b_h^{k}(s)$.
                    
                    \STATE $B_{h}^{k}(s,a) =b_{h}^{k}(s) + \max_{p' \in \calP^k_h(s,a)} \sum_{s' \in \calS,a' \in \calA} p'_h(s'\mid s,a) \pi_{h+1}^{k}(a' \mid s') \hat B^k_{h+1}(s',a')$.
                    
                \ENDFOR
            \ENDFOR
            \STATE {\color{gray} \# Policy Improvement}
            \STATE Define the policy $\pi^{k+1}$ for every $(s,a,h) \in \calS \times \calA \times [H]$ by:
            \[
                \pi^{k+1}_h(a \mid s) = \frac{\pi^k_h(a \mid s)\exp \l(-\eta( \hat \aggQ_{h}^{k}(s,a) - B_h^k(s,a) ) \r) } 
                {\sum_{a' \in \calA} \pi^k_h(a' \mid s) \exp \l(-\eta( \hat \aggQ_{h}^{k}(s,a') - B_h^k(s,a') ) \r)}.
            \]
        \ENDFOR
    \end{algorithmic}
\end{algorithm}

\begin{theorem}[Restatement of \Cref{theorem:unknown}]
    Under unknown dynamics,
    running \Cref{alg:tabular-unknown-p} with $\eta=\frac{\sqrt{\logterm}}{H \sqrt{S A K} + H^2\sqrt{K}}$ and $\gamma = 2\eta H$, guarantees with probability $1-\delta$,
    \begin{align*}
        R_K \lesssim  H^{3}S\sqrt{AK}\logterm 
                        + H^{4}S^{3}A\logterm^{2}.
    \end{align*}
\end{theorem}

\begin{proof}
    Using \Cref{eq: regret U}, we can break the regret of the algorithm as,
    \begin{align*}
        R_{K} &  = \sum_{k = 1}^{K} \sum_{h, s}\occ_{h}^{\star}(s) \l\langle \pi_{h}^{k}(\cdot\mid s) - \pi_{h}^{\star}(\cdot\mid s), U_{h}^{k}(s, \cdot) \r\rangle
            \\
            &  
            = \underbrace{\sum_{k = 1}^{K} \sum_{h, s}\occ_{h}^{\star}(s) \l\langle \pi_{h}^{k}(\cdot\mid s), \aggQ_{h}^{k}(s, \cdot) - \hat\aggQ_{h}^{k}(s, \cdot) \r\rangle }_{\textsc{Bias}_{1}}
            + \underbrace{\sum_{k = 1}^{K} \sum_{h, s}\occ_{h}^{\star}(s) \l\langle \pi_{h}^{\star}(\cdot\mid s), \hat\aggQ_{h}^{k}(s, a) - \aggQ_{h}^{k}(s, \cdot) \r\rangle }_{\textsc{Bias}_{2}}
            \\
            & \quad + \underbrace{\sum_{k = 1}^{K} \sum_{h, s}\occ_{h}^{\star}(s) \l\langle \pi_{h}^{k}(\cdot\mid s) - \pi_{h}^{\star}(\cdot\mid s), \hat\aggQ_{h}^{k}(s, \cdot) - B_{h}^{k}(s, \cdot) \r\rangle }_{\textsc{Reg}}
            \\
            & \qquad 
            + \underbrace{\sum_{k = 1}^{K} \sum_{h, s}\occ_{h}^{\star}(s) \l\langle \pi_{h}^{k}(\cdot\mid s) - \pi_{h}^{\star}(\cdot\mid s), B_{h}^{k}(s, \cdot) \r\rangle }_{\textsc{Bonus}},
    \end{align*}
\end{proof}
Due to optimism of the estimator $\textsc{Bias}_{2}\leq\tilde{O}(H^{2}/\gamma)$
under the good event $G_3$.
In \cref{lemma: bias1-unknown,lemma:reg-unknown,lemma:bonus-unknown} we bound $\textsc{Bias}_1$, $\textsc{Reg}$ and $\textsc{Bonus}$. Overall we get that,

\begin{align*}
    R_{K} &  \leq \underbrace{\frac{2}{3}
        \sum_{k = 1}^{K} 
            \sum_{h, s}\occ_{h}^{\star}(s)\tilde{b}_{h}^{k}(s) 
                + \sum_{k = 1}^{K} \sum_{h, s}\occ_{h}^{\star}(s)\bar{b}_{h}^{k}(s) 
                    + O\l(\frac{H^{2}}{\gamma}\logterm\r)}_{\textsc{Bais}_{1}} 
                        + \underbrace{O\l(\frac{H^{2}}{\gamma}\logterm\r)}_{\textsc{Bais}_{2}}
        \\
        & \qquad + \underbrace{\frac{H\ln A}{\eta} 
            + \frac{1}{3}\sum_{k, h, s}\occ_{h}^{\star}(s)\tilde{b}_{h}^{k}(s) 
                + O\l(\eta H^{5}K + \frac{H^{2}}{\gamma}\logterm\r)}_{\textsc{Reg}}
        \\
        & \qquad + \underbrace{O\l(\gamma H^{2}SAK 
            + H^{3}S\sqrt{AK}\logterm 
                + H^{4}S^{3}A\logterm^{2}\r) 
                    - \sum_{k = 1}^{K} \sum_{h, s}\occ_{h}^{\star}(s)b_{h}^{k}(s)}_{\textsc{Bonus}}
        \\
        &  \leq O\Bigg(\frac{H\ln A}{\eta} 
            + \gamma H^{2}SAK 
                + \eta H^{5}K 
                    + \frac{H^{2}}{\gamma}\logterm 
                        + H^{3}S\sqrt{AK}\logterm 
                            + H^{4}S^{3}A\logterm^{2}\Bigg)
\end{align*}

Plugging $\eta$ and $\gamma$ completes the proof.

\subsection{Good event}

We define the following good event $G = \bigcap_{i=1}^5 G_i$ which holds with high probability (\cref{lemma:good event unknown}):

\begin{align*}
    G_1
    & = 
    \l\{ \forall (k,s',s,a,h). \ 
    \l| p_{h}(s'\mid s,a)-\bar{p}_{h}^{k}(s'\mid s,a) \r|
    \leq 
    4\sqrt{ \frac{\bar{p}_{h}^{k}(s' \mid s,a)  \log\frac{6 HSAK}{\delta}}{\max \{ n_{h}^{k}(s,a), 1 \}}} + 10 \frac{\log\frac{6 HSAK}{\delta}}{\max \{ n_{h}^{k}(s,a), 1 \}}
    \r\}
    \\
    G_2
    & = 
    \l\{ \sum_{h,s,a,k}|\ol \occ_{h}^{k}(s,a) - \ul \occ_{h}^{k}(s,a)|
        \le O \l( \sqrt{ H^{4} S^{2} A K \log \frac{6 KHSA}{\delta}} + H^3 S^{3} A \log^2 \frac{6 KHSA}{\delta}  \r)
    \r\}
    \\
    G_3 &= \l\{ \sum_{k=1}^K \sum_{h, s, a}\occ_{h}^{\star}(s)\pi_{h}^{k}(a\mid s) \l(\bbE_{k} \l[\hat\aggQ_{h}^{k}(s, a) \r] - \hat\aggQ_{h}^{k}(s, a) \r) 
        \leq \frac{1}{3}\sum_{k = 1}^{K} \sum_{h, s}\occ_{h}^{\star}(s) \tilde b_{h}^{k}(s) + \frac{H^{2}\log\frac{6}{\delta}}{\gamma} 
            \r\}
    \\
    G_4 &= \l\{ \sum_{k = 1}^{K} \sum_{h, s, a} \frac{\occ_{h}^{\star}(s)\pi_{h}^{k}(a\mid s) \mathbb{I}\{s_{h}^{k} = s, a_{h}^{k} = a\}}{(\ol\occ_{h}^{k}(s, a) + \gamma)^{2}} - \frac{\occ_{h}^{\star}(s)\pi_{h}^{k}(a\mid s)}{\ol\occ_{h}^{k}(s, a) + \gamma} \leq \frac{H\ln\frac{6H}{\delta}}{2\gamma^{2}} \r\}
    \\
    G_5 &= \l\{
        \sum_{k = 1}^{K} \sum_{h, s}\occ_{h}^{\star}(s) \l\langle \pi_{h}^{\star}(\cdot\mid s), \hat\aggQ_{h}^{k}(s, \cdot) - \aggQ_{h}^{k}(s, \cdot) \r\rangle \leq \frac{H^2}{2\gamma}\ln\frac{6H}{\delta}
            \r \}
\end{align*}

Under the good event the regret will be deterministically bounded.

Event $G_1$ holds with high probability by standard Bernstein inequality (see, e.g., Lemma 2 in \citet{jin2019learning}).
As a consequence of event $G_1$, $p\in \calP^k$ for all $k$. In particular, $\ol \occ^k_h(s,a) \geq \occ^k_h(s,a)$ for all $k,h,s$ and $a$.
$G_2$ Holds with high probability by a 
standard techniques by \citet{jin2019learning} of summing the confidence radius on the trajectory.
$G_3, G_4$ and $G_5$ follow similar techniques as in \citet{luo2021policy}, adapted to our case.

\begin{lemma}[Event $G_3$]
    \label{lemma:good bias1 unknown}
    With probability $1-\delta$,
    \begin{align*}
        \sum_{k=1}^K \sum_{h, s, a}\occ_{h}^{\star}(s)\pi_{h}^{k}(a\mid s) \l(\bbE_{k} \l[\hat\aggQ_{h}^{k}(s, a) \r] - \hat\aggQ_{h}^{k}(s, a) \r) 
            \leq \frac{1}{3}\sum_{k = 1}^{K} \sum_{h, s}\occ_{h}^{\star}(s)b_{h}^{k}(s) + \frac{H^{2}\log\frac{1}{\delta}}{\gamma}
    \end{align*}

\end{lemma}

\begin{proof}
    Similar to \cref{lemma:good bias1},
    let $Y_{k}=\sum_{h,s} \occ_{h}^{\star}(s) \l\langle \pi_{h}^{k}(\cdot\mid s),\hat\aggQ_{h}^{k}(s,\cdot) \r\rangle $.
    To bound $\sum_{k}\mathbb{E}_{k}[Y_{k}]-\sum_{k}Y_{k}$ we'll use
    a form of Freedman's Inequality (\cref{lemma:freedman}). For that we need to bound $\mathbb{E}_{k}[Y_{k}^{2}]$:
    \begin{align*}
        \mathbb{E}_{k}[Y_{k}^{2}] &  =  \mathbb{E}_{k} \l[ \l(\sum_{h, s, a}\occ_{h}^{\star}(s)\pi_{h}^{k}(a\mid s)\hat\aggQ_{h}^{k}(s, a) \r)^{2} \r]
            \\
            &  \leq  \mathbb{E}_{k} \l[ \l(\sum_{h, s, a}\occ_{h}^{\star}(s)\pi_{h}^{k}(a\mid s) \r) \l(\sum_{h, s, a}\occ_{h}^{\star}(s)\pi_{h}^{k}(a\mid s) \l(\hat\aggQ_{h}^{k}(s, a) \r)^{2} \r) \r]
              \qquad  \tag{Cauchy–Schwarz}
            \\
            &  =  H \mathbb{E}_{k} \l[\sum_{h, s, a}\occ_{h}^{\star}(s)\pi_{h}^{k}(a\mid s) \l(\hat\aggQ_{h}^{k}(s, a) \r)^{2} \r]
            \\
            &  =  H \mathbb{E}_{k} \l[\sum_{h, s, a}\occ_{h}^{\star}(s)\pi_{h}^{k}(a\mid s)\frac{(L^{k}_{1:H})^{2}}{(\ol \occ_{h}^{k}(s, a) + \gamma)^{2}} \mathbb{I}_{h}^{k}(s, a) \r]
            \\
            &  \leq  H^{3} \mathbb{E}_{k} \l[\sum_{h, s, a}\occ_{h}^{\star}(s)\pi_{h}^{k}(a\mid s)\frac{ \mathbb{I}_{h}^{k}(s, a)}{(\ol\occ_{h}^{k}(s, a) + \gamma)^{2}} \r] 
            \tag{$L_{1:H}^k \leq H$}
            \\
            &  =  H^{3}\sum_{h, s, a}\occ_{h}^{\star}(s)\pi_{h}^{k}(a\mid s)\frac{\occ_{h}^{k}(s, a)}{(\ol\occ_{h}^{k}(s, a) + \gamma)^{2}}
            \tag{$\bbE_k[\bbI_h^k(s,a)] = \occ_h^k(s,a)$}
            \\
            &  \leq  H^{3}\sum_{h, s, a}\occ_{h}^{\star}(s)\frac{\pi_{h}^{k}(a\mid s)}{\ol\occ_{h}^{k}(s, a) + \gamma}.
            \tag{$\occ_{h}^{k}(s, a) \leq \ol\occ_{h}^{k}(s, a)$ under $G_1$}
    \end{align*}
    By \cref{lemma:freedman} with probability $1-\delta$,
    \[
    \sum_{k} \mathbb{E}_{k}[Y_{k}] - \sum_{k}Y_{k} \leq \alpha\sum_{k} \mathbb{E}_{k}[Y_{k}^{2}] + \frac{\log\frac{1}{\delta}}{\alpha}
    \]
    where $\alpha\in(0,1/R]$ for $R$ such that $|Y_{k}|\leq R$ for
    any $k$. In our case, $|Y_{k}|\leq H^{2}/\gamma$. And so, 
    \[
        \sum_{k} \mathbb{E}_{k}[Y_{k}] - \sum_{k}Y_{k} \leq \frac{1}{3}\sum_{k = 1}^{K} \sum_{h, s} \occ_{h}^{\star}(s) \tilde b_{h}^{k}(s) + \frac{H^{2}\log\frac{1}{\delta}}{\gamma}.
    \]
\end{proof}

\begin{lemma}[Event $G_4$]
    \label{lemma:good reg unknown}
    With probability $1-\delta$,
    \begin{align*}
         \sum_{k = 1}^{K} \sum_{h, s, a} \frac{\occ_{h}^{\star}(s)\pi_{h}^{k}(a\mid s) \mathbb{I}\{s_{h}^{k} = s, a_{h}^{k} = a\}}{(\ol\occ_{h}^{k}(s, a) + \gamma)^{2}} - \frac{\occ_{h}^{\star}(s)\pi_{h}^{k}(a\mid s)}{\ol\occ_{h}^{k}(s, a) + \gamma}
            \leq   \frac{H\ln\frac{H}{\delta}}{2 \gamma^2} 
    \end{align*}
\end{lemma}

\begin{proof}
    Follows directly from \cref{lemma:concentration} with with $Z_{h}^{k}(s,a) = z_{h}^{k}(s,a)=\frac{\occ_{h}^{\star}(s)\pi_{h}^{k}(a\mid s)}{\ol\occ_{h}^{k}(s,a) + \gamma} \leq \frac 1 \gamma$
and $\tilde{\occ}_{h}^{k}(s,a)=\ol\occ_{h}^{k}(s,a)$.
\end{proof}

\begin{lemma}[Event $G_5$]
    \label{lemma:good bias2 unknown}
    With probability $1-\delta$,
    \begin{align*}
        \sum_{k = 1}^{K} \sum_{h, s}\occ_{h}^{\star}(s) \l\langle \pi_{h}^{\star}(\cdot\mid s), \hat\aggQ_{h}^{k}(s, \cdot) - \aggQ_{h}^{k}(s, \cdot) \r\rangle \leq \frac{H^2}{2\gamma}\ln\frac{H}{\delta}
    \end{align*}
\end{lemma}

\begin{proof}
    Similar to \cref{lemma:good bias2 unknown},
    we invoke \cref{lemma:concentration} with  $Z_{h}^{k}(s,a)
        =\occ_{h}^{\star}(s)\pi_{h}^{k}(a,s)[\mathbb{I}_{h}^{k}(s,a)L^{k}_{1:H} + (1-\mathbb{I}_{h}^{k}(s,a))\aggQ_{h}^{k}(s,a)]
            \leq H$,
    $z_{h}^{k}(s,a)
        =\mathbb{E}_{k}[Z_{h}^{k}(s,a)]
            =\occ_{h}^{\star}(s)\pi_{h}^{k}(a,s)\aggQ_{h}^{k}(s,a)$
    and $\tilde{\occ}_{h}^{k}(s,a) = \ol \occ_{h}^{k}(s,a)$, we get,
    \begin{align*}
        & \sum_{k = 1}^{K} \sum_{h, s}\occ_{h}^{\star}(s) \l\langle \pi_{h}^{\star}(\cdot\mid s), \hat\aggQ_{h}^{k}(s, \cdot) - \aggQ_{h}^{k}(s, \cdot) \r\rangle  
        \\
        &\quad =  \sum_{k = 1}^{K} \sum_{h, s, a} \frac{\occ_{h}^{\star}(s)\pi_{h}^{k}(a, s) \mathbb{I}_{h}^{k}(s, a)L^{k}_{1:H}}{\ol\occ_{h}^{k}(s, a) + \gamma} - \sum_{k = 1}^{K} \sum_{h, s, a}\occ_{h}^{\star}(s)\pi_{h}^{k}(a, s)\aggQ_{h}^{k}(s, a) 
                \leq \frac{H^2}{2\gamma}\ln\frac{H}{\delta}
    \end{align*}
\end{proof}

\begin{lemma}
    \label{lemma:good event unknown}
    Under \Cref{alg:tabular-unknown-p},
    the good event $G = \bigcap_{i=1}^5 G_i$ holds with probability of at least $1-\delta$.
\end{lemma}
\begin{proof}
    The proof directly follows from invoking \Cref{lemma:good reg unknown,lemma:good bias1 unknown,lemma:good bias2 unknown,lemma: confidence bernstein,lemma: sum radius} with $\delta / 5$ and taking the union bound.
\end{proof}

\subsection{Bound on $\textsc{Bias}_1$}

\begin{lemma}
    \label{lemma:unbiased-unknown}
    For any $h,s,a$ and $k$,
    \begin{align*}
        \mathbb{E} \l[ \hat\aggQ_{h}^{k}(s,a)\mid\pi^{k} \r]
        =
        \frac{\occ_{h}^{k}(s,a)}{\ol \occ_{h}^{k}(s,a) + \gamma} \aggQ_{h}^{k}(s,a)
    \end{align*}
\end{lemma}

\begin{proof}
    By definition $\Pr(\bbI_h^k(s,a) = 1 \mid \pi^k) = \occ_h^k(s,a)$. Using the law of total expectation and the fact that $\hat\aggQ_{h}^{k}(s, a) = 0$ whenever $\bbI_h^k(s,a) = 0$ we get,
    \begin{align*}
     \mathbb{E}_k \l[\hat\aggQ_{h}^{k}(s, a) \r] 
     & = \mathbb{E}_k \l[\hat\aggQ_{h}^{k}(s, a) \mid \bbI_h^k(s,a) = 1 \r] 
        \cdot \occ_h^k(s,a)
             + \underbrace{\mathbb{E}_k \l[\hat\aggQ_{h}^{k}(s, a) \mid \bbI_h^k(s,a) = 0 \r]}_{=0}  (1 -  \occ_h^k(s,a))
        \\
        & =
        \mathbb{E}_k \l[\frac{\sum_{h' = 1}^{H} \ell_{h'}^{k}(s_{h'}^{k}, a_{h'}^{k})}{\ol \occ_{h}^{k}(s, a) + \gamma}  \mid \mathbb{I}_{h}^{k}(s, a) = 1 \r] \cdot \occ_h^k(s,a)
        \\
            &  = \frac{\occ_{h}^{k}(s, a)}{\ol \occ_{h}^{k}(s, a) + \gamma} \mathbb{E}_k \l[\sum_{h' = h}^{H} \ell_{h'}^{k}(s_{h'}^{k}, a_{h'}^{k}) \,\Big\vert\, s_{h}^{k} = s, a_{h}^{k} = a \r]
            \\
                &  = \frac{\occ_{h}^{k}(s, a)}{\ol \occ_{h}^{k}(s, a) + \gamma}U_{h}^{k}(s, a).
                \\
                & \qedhere
    \end{align*}
\end{proof}

\begin{lemma}
    \label{lemma: bias1-unknown}
    Under the good event,
    \[
	\textsc{Bias}_{1} \leq \frac{2}{3}\sum_{k=1}^{K}\sum_{h,s}
	                  \occ_{h}^{\star}(s) \tilde b_{h}^{k}(s)
	                  + \sum_{k=1}^{K} \sum_{h,s}\occ_{h}^{\star}(s) \bar b_{h}^{k}(s) + O\l( 
	                  \frac{H^{2}}{\gamma}\logterm \r).                                              
    \]
\end{lemma}

\begin{proof}
    Let $Y_{k}= \sum_{h,s} \occ_{h}^{\star}(s) \l\langle \pi_{h}^{k}(\cdot \mid s),\hat\aggQ_{h}^{k}(s,\cdot) \r\rangle $. It holds that,
    \begin{align*}
        \textsc{Bias}_1 
        = 
        \sum_{k=1}^{K} \sum_{h,s} \occ_{h}^{\star}(s) \l\langle \pi_{h}^{k}(\cdot \mid s),\aggQ_{h}^{k}(s,\cdot) \r\rangle 
        - \sum_{k=1}^{K} \bbE_k [Y_{k} ] 
        + \sum_{k=1}^{K} \bbE_k [Y_{k} ] 
        - \sum_{k=1}^{K}Y_{k}.
    \end{align*}
Under the good event, $G_3$, it holds that
$$
    \sum_{k=1}^{K} \bbE_k [Y_{k} ] - \sum_{k=1}^{K}Y_{k} 
    \leq \frac{1}{3} \sum_{k=1}^{K} \sum_{h,s} \occ_{h}^{\star}(s) \tilde b_{h}^{k}(s) 
    + \frac{H^{2}}{\gamma} \ln \frac{10}{\delta}.
$$
Using \Cref{lemma:unbiased-unknown}, and the fact that under the good event $\ul \occ_{h}^{k}(s, a)\leq \occ_{h}^{k}(s, a)$,
\begin{align*}
    \sum_{k = 1}^{K}   \sum_{h, s}\occ_{h}^{\star}(s)\l\langle\pi_{h}^{k}(\cdot\mid s), \aggQ_{h}^{k}(s, \cdot)\r\rangle -& 
    \sum_{k = 1}^{K} \bbE_{k}[Y_{k}]  = \sum_{k = 1}^{K} \sum_{h, s, a}\occ_{h}^{\star}(s)\pi_{h}^{k}(a\mid s)\aggQ_{h}^{k}(s, a)\l(1 - \frac{\occ_{h}^{k}(s, a)}{\ol \occ_{h}^{k}(s, a) + \gamma}\r)
                                    \\
                                    &  = \sum_{k = 1}^{K} \sum_{h, s, a}\occ_{h}^{\star}(s)\pi_{h}^{k}(a\mid s)\aggQ_{h}^{k}(s, a)\frac{\gamma + \ol \occ_{h}^{k}(s, a) - \occ_{h}^{k}(s, a)}{\ol \occ_{h}^{k}(s, a) + \gamma}
                                    \\
                                    &  \leq \sum_{k = 1}^{K} \sum_{h, s, a}\occ_{h}^{\star}(s)\pi_{h}^{k}(a\mid s)\aggQ_{h}^{k}(s, a)\frac{\gamma}{\ol \occ_{h}^{k}(s, a) + \gamma}
                                    \\
                                    & \qquad + \sum_{k = 1}^{K} \sum_{h, s, a}\occ_{h}^{\star}(s)\pi_{h}^{k}(a\mid s)\aggQ_{h}^{k}(s, a)\frac{\ol \occ_{h}^{k}(s, a) - \ul \occ_{h}^{k}(s, a)}{\ol \occ_{h}^{k}(s, a) + \gamma}
                                    \tag{under the good event $\ul \occ_{h}^{k}(s, a)\leq \occ_{h}^{k}(s, a)$}
\end{align*}

The first term above is bounded by,
\begin{align*}
\sum_{k = 1}^{K} 
    \sum_{h, s, a}\occ_{h}^{\star}(s)\pi_{h}^{k}(a\mid s)\aggQ_{h}^{k}(s, a)
        \frac{\gamma}{\ol \occ_{h}^{k}(s, a) + \gamma} 
            &  \leq \sum_{k = 1}^{K} 
                \sum_{h, s, a}\occ_{h}^{\star}(s)\pi_{h}^{k}(a\mid s)
                    \frac{H\gamma}{\ol \occ_{h}^{k}(s, a) + \gamma}
            \\
            &  = \frac{1}{3}\sum_{k = 1}^{K} 
                \sum_{h, s, a}\occ_{h}^{\star}(s) \tilde{b}_{h}^{k}(s)
\end{align*}
The second term is bounded by,
\begin{align*}
    \sum_{k = 1}^{K} 
        \sum_{h, s, a}\occ_{h}^{\star}(s)
            \pi_{h}^{k}(a\mid s)\aggQ_{h}^{k}(s, a)
                \frac{\ol \occ_{h}^{k}(s, a) - \ul \occ_{h}^{k}(s, a)}{\ol \occ_{h}^{k}(s, a) + \gamma} 
                    &  \leq \sum_{k = 1}^{K} 
                        \sum_{h, s, a}\occ_{h}^{\star}(s)
                            \pi_{h}^{k}(a\mid s) H 
                                \frac{\ol \occ_{h}^{k}(s, a) - \ul \occ_{h}^{k}(s, a)}{\ol \occ_{h}^{k}(s, a) + \gamma}
                    \\
                    &  = \sum_{k = 1}^{K} 
                        \sum_{h, s, a}\occ_{h}^{\star}(s)
                            \bar{b}_{h}^{k}(s)
\end{align*}
\end{proof}

\subsection{Bound on \textsc{Reg}}
\begin{lemma}
    \label{lemma:reg-unknown}
    For $\eta \leq \frac{\gamma}{2H}$, under the good event of \cref{lemma:good reg}, 
    \begin{align*}
        \textsc{Reg} 
            &  \leq \frac{H\ln A}{\eta} 
                + 16\eta H^{5}K 
                    + \frac{1}{3}\sum_{k, h, s}\occ_{h}^{\star}(s)b_{h}^{k}(s) 
                        + O \l(\frac{H^{2}\logterm}{\gamma} \r)
    \end{align*}
\end{lemma}
\begin{proof}
    Using standard entropy regularized OMD guarantee (\cref{lemma:OMD}),
    \begin{align*}
        \textsc{Reg} 
            & \leq \frac{H\ln A}{\eta} + 2\eta\sum_{k = 1}^{K} \sum_{h, s, a}\occ_{h}^{\star}(s)\pi_{h}^{k}(a\mid s) \l(\hat\aggQ_{h}^{k}(s, a) - \hat B_{h}^{k}(s, a) \r)^{2}
            \\
                & \leq \frac{H\ln A}{\eta} + 2\eta\sum_{k = 1}^{K} \sum_{h, s, a} \occ_{h}^{\star}(s)\pi_{h}^{k}(a\mid s)\hat\aggQ_{h}^{k}(s, a)^{2} + 16\eta H^{5}K,
    \end{align*}
    where in the second inequality we use the fact that $b_h^k(s) \leq 4H$ and thus, $\hat B_h^k(s,a) \leq 4H^2$.
    For the middle term,
    \begin{align*}
        2\eta\sum_{k, h, s, a}\occ_{h}^{\star}(s)
         \pi_{h}^{k}(a\mid s)
            \hat\aggQ_{h}^{k}(s, a)^{2} 
                &  \leq 2\eta\sum_{k, h, s, a}\occ_{h}^{\star}(s)\pi_{h}^{k}(a\mid s)\frac{H^{2} \mathbb{I}\{s_{h}^{k} = s, a_{h}^{k} = a\}}{(\ol \occ_{h}^{k}(s, a) + \gamma)^{2}}
                \\
                &  \leq 2\eta H^{2}\sum_{k, h, s, a} \frac{\occ_{h}^{\star}(s)\pi_{h}^{k}(a\mid s)}{\ol \occ_{h}^{k}(s, a) + \gamma} + O \l(\eta\frac{H^{3}\logterm}{\gamma^{2}} \r)
                \\
                &  = \frac{2\eta}{3\gamma}H\sum_{k, h, s}\occ_{h}^{\star}(s) \tilde b_{h}^{k}(s) + O \l(\eta\frac{H^{3}\logterm}{\gamma^{2}} \r)
                \\
                &  \leq \frac{1}{3}\sum_{k, h, s}\occ_{h}^{\star}(s) \tilde b_{h}^{k}(s) + O \l( \frac{H^{2}\logterm}{\gamma} \r).
    \end{align*}
    The second inequality above is under the good event $G_4$, and the last inequality is since $\eta \leq \frac{H}{2 \gamma}$.
\end{proof}

\subsection{Bound on $\textsc{Bonus}$}

\begin{lemma}
    \label{lemma:value diff bounus}
    Under the good event,
    $$
        \textsc{Bonus}  \leq  \sum_{k,h,s}\ol\occ_{h}^{k}(s)b_{h}^{k}(s)
                    - \sum_{k,h,s}\occ_{h}^{\star}(s)b_{h}^{k}(s).
    $$
\end{lemma}

\begin{proof}
    Let $\hat p^k$ be the transition function chosen by the algorithm when calculating $\hat B^k$. It holds that
    \begin{align*}
        \sum_{h,s} & \occ_{h}^{*}(s)\l\langle \pi_{h}^{k}(\cdot\mid s)-\pi_{h}^{*}(\cdot\mid s), \hat B_{h}^{k}(s,\cdot)\r\rangle
        =
                    \\
                    & =
                    \sum_{h,s,a} \occ_{h}^{*}(s) \pi_{h}^{k}(a \mid s) \hat B_{h}^{k}(s, a) - \sum_{h,s,a} \occ_{h}^{*}(s) \pi_{h}^{*}(a \mid s) \hat B_{h}^{k}(s, a)
                    \\
                    & =
                    \sum_{h,s,a} \occ_{h}^{*}(s) \pi_{h}^{k}(a \mid s) \hat B_{h}^{k}(s, a) 
                    \\
                    & \qquad
                    - 
                    \sum_{h,s,a} \occ_{h}^{*}(s) \pi_{h}^{*}(a \mid s) \Bigg( b_{h}^{k}(s) +  \sum_{s',a'}  \hat p^k_h(s'\mid s,a) \pi^k_{h+1}(a' \mid s') \hat B^k_{h+1}(s',a') \Bigg)
                    \\
                    & \leq
                    \sum_{h,s,a} \occ_{h}^{*}(s) \pi_{h}^{k}(a \mid s) \hat B_{h}^{k}(s, a) 
                    \\
                    & \qquad
                    - 
                    \sum_{h,s,a} \occ_{h}^{*}(s) \pi_{h}^{*}(a \mid s) \Bigg( b_{h}^{k}(s) + \sum_{s',a'}  p_h(s'\mid s,a) \pi^k_{h+1}(a' \mid s') \hat B^k_{h+1}(s',a') \Bigg)
                    \\
                    & =
                    \underbrace{\sum_{h,s,a} \occ_{h}^{*}(s) \pi_{h}^{k}(a \mid s) \hat B_{h}^{k}(s, a) 
                            -
                            \sum_{h,s,a} \occ^*_{h+1}(s) \pi^k_{h+1} (a \mid s) \hat B^k_{h+1} (s,a)}_{(i)}
                                - \sum_{h,s} \occ_{h}^{*}(s) b_{h}^{k}(s) 
                                \tag{$*$}
    \end{align*}
    where the inequality is since $p\in\calP^k$ under event $G_1$, and $\hat p^k$ maximizes the term in the parentheses. $(*)$ uses $\sum_{s,a} \occ_h^\star(s) \pi^\star_h(a\mid s) p_h(s'\mid s,a) = \occ_{h+1}^\star(s')$ and then changes the variables $s'$ and $a'$ to $s$ and $a$. $(i)$ is a telescopic sum, and recall that $\hat B_{H+1}^k \equiv 0$. Thus,
    \begin{align*}
        (i)
                     =
                    \sum_{s,a} \occ^*_1(s) \pi^k_1(a \mid s) \hat B^k_1(s,a) 
                    & =
                    \sum_{a} \pi^k_1(a \mid \sinit) \hat B^k_1(\sinit,a) 
                    \tag{$\occ^*_1(s) = \bbI\{ s = \sinit \}$}
                    \\                    
                    & =
                    V_1^{\pi^k,\hat p^k}(\sinit)
                    \\
                    & =
                    \sum_{h,s} \occ_{h}^{\pi^k,\hat p^k}(s) b_{h}^{k}(s) 
                    \\
                    & \leq
                    \sum_{h,s} \ol \occ^{k}_h(s) b^k_h(s),
    \end{align*}
where the inequality is since $\hat p^k\in\calP^k$, and $\ol \occ^{k}_h(s)$ is the maximal occupancy with respect to transitions in $\calP^k$. Plugging $(i)$ back in the last display completes the proof.
\end{proof}

\begin{lemma}
    \label{lemma:bonus-unknown}
    Under the good event,
    $$
        \textsc{Bonus}  \leq  O\l( \gamma H^2 SA K 
            +   H^{3}S\sqrt{AK}\logterm 
                + H^{4}S^{3}A\logterm^{2} \r)
                    - \sum_{k,h,s}\occ_{h}^{\star}(s)b_{h}^{k}(s).
    $$
\end{lemma}
\begin{proof}
    From lemma \Cref{lemma:value diff bounus},
    \begin{align*}
    \textsc{Bonus}
        & \leq \sum_{k,h,s} \ol\occ_{h}^{k}(s)\tilde b_{h}^{k}(s) 
            + \sum_{k,h,s} \ol\occ_{h}^{k}(s) \bar b_{h}^{k}(s) 
                - \sum_{k,h,s} \occ_{h}^{\star}(s)b_{h}^{k}(s).
    \end{align*}

    The first term is bounded by,
    \begin{align*}
        \sum_{k=1}^{K}\sum_{h,s}
            \ol\occ_{h}^{k}(s) \tilde b^k_h(s)
        & =
        3 \gamma H \sum_{k=1}^{K}\sum_{h,s,a}\frac{\ol\occ_{h}^{k}(s) \pi^{k}_h(a \mid s) }{\ol \occ_{h}^{k}(s)\pi^{k}_h(a \mid s) + \gamma}
        \leq
        3 \gamma H^2 SA K. 
    \end{align*}

    For the second term we use the good event $G_2$,
    \begin{align*}
        \sum_{k, h, s}\ol\occ_{h}^{k}(s)
            \bar{b}_{h}^{k}(s) &  = H\sum_{k, h, s}\ol\occ_{h}^{k}(s, a)
                                    \frac{\ol \occ_{h}^{k}(s, a) - \ul \occ_{h}^{k}(s, a)}{\ol \occ_{h}^{k}(s, a) + \gamma}
                                \\
                                &  \leq  H\sum_{k, h, s}
                                    |\ol \occ_{h}^{k}(s, a) - \ul \occ_{h}^{k}(s, a)|
                                \\
                                &  \leq  H^{3}S\sqrt{AK}\logterm 
                                    + H^{4}S^{3}A\logterm^{2},
    \end{align*}
    where the last inequality is by \Cref{lemma: sum radius}.
\end{proof}

\newpage
\section{Lower Bound}
\label{sec: appendix lower bound}

\begin{theorem}[Restatement of \Cref{thm:lower bound}]
\label{thm:lower bound appendix}
    Assume that $H,S,A\geq 2$ and $K \geq 2 S$. Any learning algorithm for the online MDPs with known dynamics and aggregate bandit feedback problem must incur at least $\Omega(H^2\sqrt{S A K})$ expected regret in the worst case.
\end{theorem}

\begin{proof}
    Consider an MDP with $S$ states: $s_1,s_2,\dots,s_{S}$ where $s_1$ is the initial state.
    The idea is to encode a hard multitask bandit problem with $H-1$ tasks in each of the states.
    The agent starts in the initial state $s_1$ where the loss is $0$ and any action  transitions to each of the states $s_1,\dots,s_{S}$ with probability $1/S$. I.e., $p_1(s'\mid a,s_1) = 1/S$ for any $s'$ and $a$. From time $h\geq 2$ the agent stays at the same state, $p_h(s'\mid s,a) = \bbI\{ s' = s \}$.
    Each state $s_i$ encodes a hard multitask bandit problem with $H-1$ tasks. That is, the losses are generated (independently for each state) from the (randomized) instance that attains the lower bound of \cref{lemma:lower bound multitask bandits}.

    Denote by $T_{i}$ the number of times we transition to $s_i$. From \Cref{lemma:lower bound multitask bandits} the expected regret from visits at $s_i$ is at least $\Omega(\bbE[H^2\sqrt{A T_i}])$. In total, we have a lower bound on the regret of
    \begin{align*}
        \Omega \left( \bbE \left[ H^2 \sum_{i=1}^S \sqrt{SA T_{i}}\right] \right) = \Omega \left( H^2 S \sqrt{A} \bbE [ \sqrt{X} ] \right),
    \end{align*}
    for $X \sim Bin(n = K , p = 1/S)$ since each $T_i$ is a binomial random variable with parameters $K$ and $1/S$.
    By \cref{lem:bin-rv-expected-square-root}, we have $\bbE [ \sqrt{X} ] \ge \tilde\Omega (\sqrt{n p}) = \tilde\Omega (\sqrt{K / S})$ for $K \ge 2 S$ which proves the lower bound $\Omega ( H^2\sqrt{ S A K})$.
\end{proof}

\newpage

\section{Auxiliary Lemmas}

\begin{lemma}[Azuma–Hoeffding inequality]
    \label{lemma:Azuma–Hoeffding}
    Let $\{ X_t \}_{t\geq 1}$ be a real valued martingale difference sequence adapted to a filtration $\calF_1 \subseteq \calF_2 \subseteq...$ (i.e., $\bbE[X_t \mid \calF_t] = 0$). 
    If $|X_t| \leq R$ a.s. then with probability at least $1-\delta$,
    \[
        \sum_{t=1}^T X_t 
        \leq 
         R \sqrt{T \ln\frac{1}{\delta}}.
    \]
\end{lemma}

\begin{lemma}[A special form of Freedman's Inequality, Theorem 1 of \citet{beygelzimer2011contextual}]
    \label{lemma:freedman}
    Let $\{ X_t \}_{t\geq 1}$ be a real valued martingale difference sequence adapted to a filtration $\calF_1 \subseteq \calF_2 \subseteq...$ (i.e., $\bbE[X_t \mid \calF_t] = 0$). 
    If $|X_t| \leq R$ a.s. then for any $\alpha \in (0,1/R), T \in \mathbb{N}$ it holds with probability at least $1-\delta$,
    \[
        \sum_{t=1}^T X_t 
        \leq 
        \alpha \sum_{t=1}^T \bbE[X_t^2 \mid \calF_{t}] + \frac{\log(1/\delta)}{\alpha}.
    \]
\end{lemma}

\begin{lemma}[Consequence of Freedman’s Inequality, e.g., Lemma E.2 in \citep{cohen2021minimax}]
    \label{lemma:cons-freedman}
     Let $\{ X_t \}_{t\geq 1}$ be a sequence of random variables, supported in $[0,R]$, and adapted to a filtration $\calF_1 \subseteq \calF_2 \subseteq...$. For any $T$, with probability $1-\delta$,
     \[
        \sum_{t=1}^T X_t \leq 2 \bbE[X_t \mid \calF_t] + 4R \log\frac{1}{\delta}.
     \]
     
\end{lemma}

\begin{lemma}[Lemma A.2 of \citet{luo2021policy}]
    \label{lemma:concentration}
    Given a filtration $\calF_0 \subseteq \calF_1 \subseteq \dots$, let $z^k_h(s,a)\in [0,R]$ and $\tilde \occ_h^k(s,a)\in [0, 1]$ be  sequences of $\calF_k$-measurable functions. If $Z_h^k(s,a) \in [0,R]$ is a sequence of random variables such that $\bbE[Z_h^k(s,a)\mid \calF_k] = z_h^k(s,a)$ then with probability $1-\delta$,
    \[
        \sum_{k=1}^{K}\sum_{h,s,a}\frac{\mathbb{I} \{s_{h}^{k}=s,a_{h}^{k}=s\}Z_{h}^{k}(s,a)}{\tilde{\occ}_{h}^{k}(s,a) + \gamma}
        - \sum_{k=1}^{K}\sum_{h,s,a} \frac{\occ_{h}^{k}(s,a)z_{h}^{k}(s,a)}{\tilde{\occ}_{h}^{k}(s,a)} \leq \frac{RH}{2\gamma}\ln\frac{H}{\delta}
    \]
\end{lemma}

\begin{lemma}[Standard entropy regularized OMD guarantee, see e.g., \citep{hazan2016introduction}]
\label{lemma:OMD}
    Let $\eta > 0$, and an arbitrary sequence $\{g_k\}_{k=1}^K$ such that for all $k \in [K]$, $a \in [d]$, $g_k \in \mathbb{R}^d$ and $\eta g_k(a) \geq -1$. Let $x_k \in \Delta_d$ be a sequence of vectors such that for all $a$, $x_1(a) = 1/n$, for all $k \in [K]$, $a \in [d]$,
    \[
    x_{k+1}(a) = \frac{x_k(a) e^{-\eta g_k(a)}}{\sum_{a' \in [n]} x_k(a') e^{-\eta g_k(a')}}.
    \]
    
    Then for any $x^\star \in \Delta_d$,
    \[
    \sum_{k=1}^K \langle g_k, x_k - x \rangle \leq \frac{\log d}{\eta} + \eta \sum_{k=1}^K \sum_{i=1}^d x_k(i) g_k(i)^2.
    \]
\end{lemma}

\begin{lemma}[Lemma 2 in \citet{jin2019learning}]
    \label{lemma: confidence bernstein}
    With probability $1-\delta$, for all $(k,s',s,a,h)$,
    \begin{align*}
        \l| p_{h}(s'\mid s,a)-\bar{p}_{h}^{k}(s'\mid s,a) \r|
        \leq 
        4\sqrt{ \frac{\bar{p}_{h}^{k}(s' \mid s,a)  \log\frac{HSAK}{\delta}}{\max \{ n_{h}^{k}(s,a), 1 \}}} + 10 \frac{\log\frac{HSAK}{\delta}}{\max \{ n_{h}^{k}(s,a), 1 \}}
    \end{align*}
\end{lemma}

\begin{lemma}[Lemma 4 in \citet{jin2019learning}]
    \label{lemma: sum radius}
    With probability $1-\delta$,
    \[
    \sum_{h,s,a,k}|\ol \occ_{h}^{k}(s,a) - \ul \occ_{h}^{k}(s,a)|
        \le O \l( \sqrt{ H^{4} S^{2} A K \log \frac{KHSA}{\delta}} + H^3 S^{3} A \log^2 \frac{KHSA}{\delta}  \r)
    \]
\end{lemma}

\begin{lemma}
    \label{lem:bin-rv-expected-square-root}
    Let $X \sim Bin(n,p)$ and assume that $n \ge \frac2p$.
    Then, $\bbE [\sqrt{X}] \ge \frac14 \sqrt{np}$.
\end{lemma}

\begin{proof}
    By Markov inequality we have:
    \[
        \bbE[\sqrt{X}]
        \ge
        \frac{\sqrt{np}}{2} \Pr \left[ \sqrt{X} \ge \frac{\sqrt{np}}{2} \right]
        =
        \frac{\sqrt{np}}{2} \Pr \left[ X \ge \frac{np}{4} \right]
        =
        \frac{\sqrt{np}}{2} \left( 1 - \Pr \left[ X < \frac{np}{4} \right] \right).
    \]
    Thus, it suffices to show that $\Pr \left[ X < \frac{np}{4} \right] \le 1/2$ which follows immediately from Multiplicative Chernoff bound and the assumption that $n \ge 2/p$.
\end{proof}

\end{document}